%% file: arxiv_main.tex
\documentclass[11pt]{article}

\usepackage[letterpaper,margin=1in]{geometry}
\usepackage[parfill]{parskip}

\usepackage{authblk}

\usepackage[toc,page,header]{appendix}

\usepackage[utf8]{inputenc} 
\usepackage[T1]{fontenc}    
\usepackage[colorlinks,citecolor=blue,urlcolor=blue,linkcolor=blue,linktocpage=true]{hyperref}       
\usepackage{url}            
\usepackage{booktabs}       
\usepackage{amsfonts}       
\usepackage{nicefrac}       
\usepackage{microtype}      
\usepackage{xcolor}         
\label{key}
\usepackage{thmtools}

\usepackage{comment}
\usepackage{amsmath}

\usepackage{epsfig}
\usepackage{graphicx}
\usepackage{wrapfig}
\usepackage{subfigure}
\usepackage{multirow}
\usepackage{hyperref}
\usepackage{graphicx}
\usepackage{caption}
\usepackage{array}
\usepackage{bbm}
\usepackage{color}
\usepackage{enumerate}
\usepackage{enumitem}
\setlist{leftmargin=10mm}
\usepackage{mathtools}
\usepackage{amsmath,amssymb,amsthm,bm}
\usepackage{amsfonts,graphicx}
\usepackage{mathrsfs}
\usepackage{algorithmic,algorithm}

\usepackage[authoryear]{natbib}

\newcommand{\blue}[1]{\textcolor{blue}{#1}}
\def\rdp{\mathrm{RDP}}

\def\pr{\mathrm{Pr}}

\def\P{\mathbb{P}}
\def\argmax{Argmax}

\def\R{\mathbb{R}}
\def\cA{\mathcal{A}}

\def\cD{\mathcal{D}}

\def\cM{\mathcal{M}}
\def\cN{\mathcal{N}}

\def\cR{\mathcal{R}}

\def\cT{\mathcal{T}}

\def\cX{\mathcal{X}}
\def\cY{\mathcal{Y}}

\newcommand{\argmin}{\text{argmin}}
\newcommand{\ttheta}{\tilde{\theta}}

\theoremstyle{plain}
\newtheorem{theorem}{Theorem}[section]
\newtheorem{proposition}[theorem]{Proposition}
\newtheorem{lemma}[theorem]{Lemma}

\theoremstyle{definition}
\newtheorem{definition}[theorem]{Definition}

\newtheorem{remark}[theorem]{Remark}

\theoremstyle{plain}
\newtheorem{example}[theorem]{Example}

\title{Generalized PTR: User-Friendly Recipes for Data-Adaptive Algorithms with Differential Privacy}

\author{Rachel Redberg, Yuqing Zhu, Yu-Xiang Wang \\
University of California, Santa Barbara\\
\texttt{\{rredberg, yuqingzhu, yuxiangw\}@ucsb.edu} \\
}

\begin{document}

\maketitle

\begin{abstract}
    The ``Propose-Test-Release'' (PTR) framework \citep{dwork2009differential} is a classic recipe for designing differentially private (DP) algorithms that are data-adaptive, i.e. those that  add less noise when the input dataset is ``nice''. We extend PTR to a more general setting by privately testing \emph{data-dependent privacy losses} rather than \emph{local sensitivity}, hence making it applicable beyond the standard noise-adding mechanisms, e.g. to queries with unbounded or undefined sensitivity. We demonstrate the versatility of generalized PTR using private linear regression as a case study. Additionally, we apply our algorithm to solve an open problem from “Private Aggregation of Teacher Ensembles (PATE)” \citep{papernot2017, papernot2018scalable} --- privately releasing the entire model with a delicate data-dependent analysis.
\end{abstract}

\input{introduction}
\label{section:introduction}

\input{relatedwork}

\label{section:related_work}

\input{preliminary}

\label{section:preliminary}

\input{genPTR}
\label{section:genPTR}

 \input{applications}

 \label{section:applications}

\input{conclusion}

\label{section:conclusion}

\subsection*{Acknowledgments}

The work was partially supported by NSF Award \# 2048091 and the Google Research Scholar Award. Yuqing was supported by the Google PhD Fellowship.

\newpage
\onecolumn

\appendix

\input{appendix}
 \input{pate_appendix}
\bibliographystyle{plainnat}
\bibliography{gen_ptr}

\end{document}

%% file: introduction.tex
\section{Introduction}
\label{sec:introduction}
The guarantees of differential privacy (DP) \citep{dwork2006calibrating} are based on worst-case outcomes across all possible datasets. A common paradigm is therefore to add noise scaled by the \emph{global sensitivity} of a query $f$, i.e. the maximum change in $f$ between any pair of neighboring datasets.

A given dataset $X$ might have a \emph{local sensitivity} that is much smaller than the global sensitivity, in which case we can hope to add a smaller amount of noise (calibrated to the local rather than the global sensitivity) while achieving the same privacy guarantee. However, this must not be undertaken na\"{i}vely -- the local sensitivity is a dataset-dependent function and so calibrating noise to the local sensitivity could leak information about the dataset \citep{nissim2007smooth}.

The ``Propose-Test-Release'' (PTR) framework \citep{dwork2009differential} resolves this issue by introducing a test to privately check whether a proposed bound on the local sensitivity is valid. Only if the test ``passes'' is the output released with noise calibrated to the proposed bound on the local sensitivity. 

PTR is a powerful and flexible tool for designing data-adaptive DP algorithms, but it has several limitations. First, it applies only to noise-adding mechanisms which calibrate noise according to the sensitivity of a query. Second, the test in ``Propose-Test-Release'' is computationally expensive for all but a few simple queries such as privately releasing the median or mode.  Third, while some existing works\blue{~\citep{decarolis2020end, kasiviswanathan2013analyzing, liu2021differential}} follow the approach of testing ``nice'' properties of a dataset before exploiting these properties in a private release to PTR \footnote{We refer to these as PTR-like methods. }, 
there has not been a systematic recipe for \emph{discovering} which properties should be tested.
    

In this paper, we propose a generalization of PTR which addresses these limitations. The centerpiece of our framework is a differentially private test on the \emph{data-dependent privacy loss}.
This test does not directly consider the local sensitivity of a query and is therefore not limited to additive noise mechanisms. Moreover, in many cases, the test can be efficiently implemented by privately releasing a high-probability upper bound, thus avoiding the need to search an exponentially large space of datasets. Furthermore, the derivation of the test itself often spells out exactly what properties of the input dataset need to be checked, which streamlines the design of data-adaptive DP algorithms.


Our contributions are summarized as follows:
\begin{enumerate}
    \item We propose a generalization of PTR which can handle algorithms beyond noise-adding mechanisms. Generalized PTR allows us to plug in \emph{any} data-dependent DP analysis to construct a high-probability DP test that adapts to favorable properties of the input dataset -- without painstakingly designing each test from scratch.
    \item We demonstrate that many existing examples of PTR and PTR-like algorithms can be unified under the generalized PTR framework, sometimes resulting in a tighter analysis (see an example of report-noisy-max in Sec~\ref{sec:binary_vote}).   
    \item We show that one can publish a DP model through privately upper-bounding a one-dimensional statistic --- no matter how complex the output space of the mechanism is. We apply this result to solve an open problem from PATE \citep{papernot2017, papernot2018scalable}.    
    \item Our results broaden the applicability of private hyper-parameter tuning \citep{liu2019private,papernot2021hyperparameter} in enabling joint-parameter selection of  DP-specific parameters (e.g., noise level) and native parameters of the algorithm (e.g., learning rate, regularization weight), which may jointly affect the data-dependent DP losses.    
\end{enumerate}

%% file: relatedwork.tex
\section{Related Work}
\label{sec:related_work}

\textbf{Data-dependent DP algorithms.}
Privately calibrating noise to the local sensitivity is a well-studied problem. One approach is to add noise calibrated to the smooth sensitivity \citep{nissim2007smooth}, an upper bound on the local sensitivity which changes slowly between neighboring datasets. 
An alternative to this -- and the focus of our work -- is  Propose-Test-Release (PTR) \citep{dwork2009differential}, which works by calculating the distance $\mathcal{D}_{\beta}(X)$ to the nearest dataset to $X$ whose local sensitivity violates a proposed bound $\beta$. The PTR algorithm then adds noise to $\mathcal{D}_{\beta}(X)$ before testing whether this privately computed distance is sufficiently large.

PTR spin-offs abound. Notable examples include stability-based methods \citep{thakurta2013differentially} (stable local sensitivity of $0$ near the input data) and privately releasing upper bounds of local sensitivity \citep{kasiviswanathan2013analyzing,liu2021differential, decarolis2020end}. We refer readers to Chapter 3 of \citet{vadhan2017complexity} for a concise summary of these classical results. 
Recent work
\citep{wang2022renyi} has provided R\'enyi DP bounds for PTR and demonstrated its applications to robust DP-SGD. Our work (see Section~\ref{subsections:pate}) also considers applications of PTR in data-adaptive private deep learning: Instead of testing the local sensitivity of each gradient step as in \citet{wang2022renyi}, our PTR-based PATE algorithm tests the data-dependent privacy loss as a whole.

\citet{liu2021differential} proposed a new variant called High-dimensional Propose-Test-Release (HPTR). HPTR provides a systematic way of solving DP statistical estimation problems by using the exponential mechanism (EM) with carefully constructed scores based on certain one-dimensional robust statistics, which have stable local sensitivity bounds. HPTR focuses on designing data-adaptive DP mechanisms from scratch; our method, in contrast, converts existing randomized algorithms (including EM and even some that do not satisfy DP) into those with formal DP guarantees. Interestingly, our proposed method also depends on a one-dimensional statistic of direct interest: the data-dependent privacy loss.

\textbf{Data-dependent DP losses.} The flip side of data-dependent DP algorithms is the study of data-dependent DP losses \citep{papernot2018scalable,soria2017individual,wang2017per}, which fix the randomized algorithm but parameterize the resulting privacy loss by the specific input dataset. For example: In the simple mechanism that adds Laplace noise with parameter $b$, data-dependent DP losses are $\epsilon(X) = \Delta_{\text{LS}}(X)/b$. The data-dependent DP losses are often much smaller than the DP loss, but they themselves depend on the data and thus may reveal sensitive information; algorithms satisfying a data-dependent privacy guarantee are not formally DP with guarantees any smaller than that of the worst-case. Existing work has considered privately publishing these data-dependent privacy losses \citep{papernot2018scalable,redberg2021privately}, but notice that privately publishing these losses does not improve the DP parameter of the given algorithm. Part of our contribution is to resolve this conundrum by showing that a simple post-processing step of the privately released upper bound of $\epsilon(\text{Data})$ gives a formal DP algorithm.


\textbf{Private hyper-parameter tuning.}
Our work has a nice connection with private hyper-parameter tuning. 
Prior work~\citep{liu2019private, papernot2021hyperparameter} requires each candidate configuration to be released with the same DP (or R\'enyi DP) parameter set. Another hidden assumption is that the parameters must not be privacy-correlated
(i.e., parameter choice will not change the privacy guarantee). Otherwise we need to use the largest DP bound across all candidates.
For example, \citet{liu2019private} show that if each mechanism (instantiated with  one group of hyper-parameters) is $(\epsilon, 0)$-DP, then running a random number of mechanisms and reporting the best option satisfies $(3\epsilon, 0)$-DP. 
Our work directly generalizes the above results by (1) considering a wide range of hyper-parameters, either privacy-correlated or not; and (2) requiring only that individual candidates to have a \emph{testable} data-dependent DP.

%% file: preliminary.tex
\section{Preliminaries}
\label{sec:preliminaries}
Datasets $X, X' \in \mathcal{X}$ are neighbors if they differ by no more than one datapoint -- i.e., $X \simeq X'$ if $d(X, X') \leq 1$. We will define $d(\cdot)$ to be the number of coordinates that differ between two datasets of the same size $n$: $d(X, Y) = \#\{i \in [n]: X_i \neq Y_i  \}$.

We use $||\cdot||$ to denote the radius of the smallest Euclidean ball that contains the input set, e.g. $||\mathcal{X}|| = \sup_{x \in \mathcal{X}} ||x||$.

The parameter $\phi$ denotes the privacy parameters associated with a mechanism (e.g. noise level, regularization). $\mathcal{M}_{\phi}$ is a mechanism parameterized by $\phi$.
For mechanisms with continuous output space, we will take $\text{Pr}[\mathcal{M}(X) = y]$ to be the probability density function of $\mathcal{M}(X)$ at $y$.

    \begin{definition}[Differential privacy \citep{dwork2006calibrating}] \label{def:dp}
        Fix $\epsilon, \delta \geq 0$. 
A randomized algorithm $\mathcal{M}: \mathcal{X} \rightarrow \mathcal{S}$ satisfies $(\epsilon, \delta)$-DP if for all neighboring datasets $X \simeq X'$ and for all measurable sets $S \subset \mathcal{S}$, 
            \[\text{Pr}\big[\mathcal{M}(X) \in S\big] \leq e^{\epsilon}\text{Pr}\big[\mathcal{M}(X') \in S\big] + \delta.\]
    \end{definition}


Suppose we wish to privately release the output of a real-valued function $f: \mathcal{X} \rightarrow \mathcal{R}$. We can do so by calculating the \emph{global sensitivity} $\Delta_{GS}$, calibrating the noise scale to the global sensitivity and then adding sampled noise to the output.

\begin{definition}[Local / Global sensitivity]
The local $\ell_*$-sensitivity of a function $f$ is defined as $\Delta_{LS}(X) = \max\limits_{X \simeq X'} || f(X) - f(X') ||_* $ and the global sensitivity of $f$ is $\Delta_{GS} = \sup_X \Delta_{LS}(X)$.
\end{definition}
\subsection{Propose-Test-Release}
Calibrating the noise level to the local sensitivity $\Delta_{LS}(X)$ of a function would allow us to add less noise and therefore achieve higher utility for releasing private queries. However, the local sensitivity is a data-dependent function and na\"ively calibrating the noise level to $\Delta_{LS}(X)$ will not satisfy DP.

PTR resolves this issue in a three-step procedure: \textbf{propose} a bound on the local sensitivity, privately \textbf{test} that the bound is valid (with high probability), and if so calibrate noise according to the bound and \textbf{release} the output.

PTR privately computes the distance $\cD_{\beta}(X)$ between the input dataset $X$ and the nearest dataset $X''$ whose local sensitivity exceeds the proposed bound $\beta$:
\begin{align*}
    \cD_{\beta}(X) = \min\limits_{X''} \{ d(X, X''): \Delta_{LS}(X'')> \beta \}.
\end{align*}

\begin{figure}[H]
\vspace{-1.4em}
\centering
\begin{algorithm}[H]
\caption{Propose-Test-Release \citep{dwork2009differential}}
\label{alg:classic_ptr}
\begin{algorithmic}[1]
\STATE{\textbf{Input}: Dataset $X$; privacy parameters $\epsilon,\delta$; proposed bound $\beta$ on $\Delta_{LS}(X)$; query function $f: \mathcal{X} \rightarrow \mathbb{R}$.}

\STATE{\textbf{if} $\cD_{\beta}(X) + \text{Lap}\left(\frac{1}{\epsilon}\right) \leq \frac{\log(1/\delta)}{\epsilon}$ \textbf{then} output $\perp$,}
\STATE{\textbf{else} release $f(X) + \text{Lap}\left(\frac{\beta}{\epsilon}\right)$.}
\end{algorithmic}
\end{algorithm}
\end{figure}
\begin{theorem} 
Algorithm~\ref{alg:classic_ptr} satisfies ($2 \epsilon, \delta$)-DP.
\citep{dwork2009differential}
\end{theorem}
Rather than proposing an arbitrary threshold $\beta$, one can also privately release an upper bound of the local sensitivity and calibrate noise according to this upper bound. This was used for node DP in graph statistics \citep{kasiviswanathan2013analyzing}, and for fitting topic models using spectral methods \citep{decarolis2020end}.

%% file: genPTR.tex
\section{Generalized PTR}
\label{sec:gen_ptr}
This section introduces the generalized PTR framework. We first formalize the notion of \emph{data-dependent} differential privacy that conditions on an input dataset $X$. 

\begin{definition}[Data-dependent privacy]
\label{def:data_dep_dp}
Suppose we have $\delta > 0$ and a function $\epsilon: \mathcal{X} \rightarrow \mathbb{R}$. We say that mechanism $\mathcal{M}$ satisfies ($\epsilon(X), \delta$) data-dependent DP\footnote{We will sometimes write that $\cM(X)$ satisfies $\epsilon(X)$ data-dependent DP with respect to $\delta$.} for dataset $X$ if for all possible output sets $S$ and neighboring datasets $X'$,
\begin{align*}
    \text{Pr}\big[\mathcal{M}(X) \in S\big] &\leq e^{\epsilon(X)}\text{Pr}\big[\mathcal{M}(X') \in S\big] + \delta, \\
       \text{Pr}\big[\mathcal{M}(X') \in S\big] &\leq e^{\epsilon(X)}\text{Pr}\big[\mathcal{M}(X) \in S\big] + \delta.
\end{align*}
\end{definition}

In generalized PTR, we propose a value $\phi$ for the randomized algorithm $\cM$, which could be a noise scale or regularization parameter -- or a set including both. For example, $\phi = (\lambda, \gamma)$ in Example~\ref{exp: posterior}. We then say that $\cM_{\phi}$ is the mechanism $\mathcal{M}$ parameterized by $\phi$, and $\epsilon_{\phi}(X)$ its data-dependent DP.

The following example illustrates how to derive the data-dependent DP for a familiar friend -- the Laplace mechanism.

\begin{example}(\emph{Data-dependent DP of Laplace Mechanism.}) \label{examp:lap_mech}
Given a function $f: \mathcal{X} \rightarrow \mathbb{R}$, we will define
\begin{align*}
    \mathcal{M}_{\phi}(X) = f(X) + \text{Lap}\left(\phi\right).
\end{align*}
We then have
\begin{align*}
    \log \dfrac{\emph{Pr}[\mathcal{M}_{\phi}(X) = y]}{\emph{Pr}[\mathcal{M}_{\phi}(X') = y]} &\leq \dfrac{|f(X) - f(X')|}{\phi}.
\end{align*}
Maximizing the above calculation over all possible outputs $y$ and using Definition~\ref{def:data_dep_dp},
\begin{align*}
    \epsilon_{\phi}(X) = \max\limits_{X': X' \simeq X} \frac{|f(X) - f(X')|}{\phi} = \frac{\Delta_{LS}(X)}{\phi}.
\end{align*}
\end{example}

The data-dependent DP $\epsilon_\phi(X)$ is a function of both the dataset $X$ and the parameter $\phi$. Maximizing $\epsilon_\phi(X)$ over $X$ recovers the standard DP guarantee of running $\cM$ with parameter $\phi$.


 
\begin{figure}[H]
\centering
\begin{algorithm}[H]
\caption{Generalized Propose-Test-Release}
\label{alg:gen_ptr}
\begin{algorithmic}[1]
\STATE{\textbf{Input}: Dataset $X$; mechanism $\cM_\phi: \mathcal{X} \rightarrow \cR$ and its privacy budget $\epsilon, \delta$; $(\hat{\epsilon}, \hat{\delta})$-DP test $\mathcal{T}$; false positive rate $\leq \delta'$; data-dependent DP function $\epsilon_\phi(\cdot)$ w.r.t. $\delta$}.
\STATE{\textbf{if not} $\mathcal{T}(\cX)$ \textbf{then} output $\perp$,}\vspace{-1pt}
\STATE{\textbf{else} release $\theta = \mathcal{M}_{\phi}(X).$}
\end{algorithmic}
\end{algorithm}
\end{figure}

\begin{theorem}[Privacy guarantee of generalized PTR]
\label{thm:gen_ptr}
Consider a proposal $\phi$ and a data-dependent DP function $\epsilon_{\phi}(X)$ w.r.t. $\delta$. Suppose that we have an ($\hat{\epsilon}, \hat{\delta}$)-DP test $\cT: \cX \rightarrow \{0, 1\}$ such that when $\epsilon_{\phi}(X) > \epsilon$, \end{theorem}
\vspace{-8mm}
\begin{align*}
    \cT(X) =
    \begin{cases}
        0  \text{ \:with probability } 1 - \delta', \\
        1  \text{\: with probability }  \delta'. 
    \end{cases}
\end{align*}
\textit{Then Algorithm~\ref{alg:gen_ptr} satisfies ($\epsilon + \hat{\epsilon}, \delta +  \hat{\delta} + \delta'$)-DP. }
 \begin{proof}[Proof sketch]
 There are three main cases to consider:
\begin{enumerate}
    \item We decide not to run $\mathcal{M}_{\phi}$.   
    \item We decide to run $\mathcal{M}_{\phi}$ and $\epsilon_{\phi}(X) > \epsilon$;    
    \item We decide to run $\mathcal{M}_{\phi}$ and $\epsilon_{\phi}(X) \leq \epsilon$.    
\end{enumerate}
In the first case, the decision to output $\perp$ is post-processing of an $(\hat{\epsilon}, \hat{\delta})$-DP mechanism and inherits its privacy guarantees. The second case occurs when the $(\hat{\epsilon}, \hat{\delta})$-DP test "fails" (produces a false positive) and occurs with probability at most $\delta'$. The third case is a composition of an $(\hat{\epsilon}, \hat{\delta})$-DP algorithm and an ($\epsilon, \delta$)-DP algorithm.
 \end{proof}
 


Generalized PTR is a \emph{strict} generalization of Propose-Test-Release. For some function $f$, define  $\cM_{\phi}$ and $\cT$ as follows:
\begin{align*}
&\mathcal{M}_{\phi}(X) = f(X) + \text{Lap}(\phi); \\
&\cT(X) = 
\begin{cases}
0 & \text{ if\:\: } \cD_{\beta}(X) + \text{Lap}\left(\frac{1}{\epsilon}\right) > \frac{\log(1/\delta)}{\epsilon},\\
1 & \text{ otherwise.} \\
\end{cases}
\end{align*}
Notice that our choice of parameterization is $\phi = \frac{\beta}{\epsilon}$, where $\phi$ is the scale of the Laplace noise. In other words, we know from Example~\ref{examp:lap_mech} that $\epsilon_{\phi}(X) > \epsilon$ exactly when $\Delta_{LS}(X) > \beta$.

For noise-adding mechanisms such as the Laplace mechanism, the sensitivity is proportional to the privacy loss (in both the global and local sense, i.e. $\Delta_{GS} \propto \epsilon$ and $\Delta_{LS} \propto \epsilon(X)$). Therefore for these mechanisms the only difference between privately testing the local sensitivity (Algorithm~\ref{alg:classic_ptr}) and privately testing the data-dependent DP (Theorem~\ref{thm:gen_ptr}) is a change of parameterization.
\subsection{Limitations of local sensitivity}
\input{local_sensitivity_examples}
\label{subsections:local_sensitivity}
\subsection{Which $\phi$ to propose}
\input{which_phi}

\label{subsections:which_phi}

\input{hyperparameter_selection}
\subsection{Construction of the DP test}
\input{test_construction}

\label{subsections:test_construction}

%% file: local_sensitivity_examples.tex
Why do we want to generalize PTR beyond noise-adding mechanisms? Compared to classic PTR, the generalized PTR framework allows us to be more flexible in both the type of test conducted and also the type of mechanism whose output we wish to release. For many mechanisms, the local sensitivity either does not exist or is only defined for specific data-dependent quantities (e.g., the sensitivity of the score function in the exponential mechanism) rather than the mechanism's output. 

The following example illustrates this issue.

\begin{example}[Private posterior sampling]\label{exp: posterior}
Let $\cM: \cX\times \cY \to \Theta $ be a private posterior sampling   mechanism~\citep{minami2016differential,wang2015privacy,gopi2022private} for approximately minimizing $F_{X}(\theta)$.

$\cM$ samples $\theta \sim P(\theta)\propto e^{-\gamma(F_X(\theta)+ 0.5\lambda ||\theta||^2)}$ with parameters $\gamma, \lambda$. Note that $\gamma,\lambda$ cannot be appropriately chosen for this mechanism to satisfy DP without going through a sensitivity calculation of $\arg\min F_X(\theta)$. In fact, the global and local sensitivity of the minimizer is unbounded even in linear regression problems, i.e when $F_X(\theta) = \frac{1}{2}||y-X\theta||^2.$ 

\end{example}
Output perturbation algorithms do work for the above problem when we regularize, but they are known to be suboptimal in theory and in practice \citep{chaudhuri2011differentially}. In Section~\ref{subsections:private_linear_regression} we demonstrate how to apply generalized PTR to achieve a data-adaptive posterior sampling mechanism.

Even in the cases of noise-adding mechanisms where PTR seems to be applicable, it does not lead to a tight privacy guarantee. Specifically, by an example of privacy amplification by post-processing (Example~\ref{exp: binary_vote} in the appendix), we demonstrate that the local sensitivity does not capture all sufficient statistics for data-dependent privacy analysis and thus is loose.




%% file: which_phi.tex
The main limitation of generalized PTR is that one needs to ``propose'' a good guess of parameter $\phi$.  Take the example of $\phi$ being the noise level in a noise-adding mechanism. Choosing too small a $\phi$ will result in a useless output $\perp$, while choosing too large a $\phi$ will add more noise than necessary. Finding this 'Goldilocks' $\phi$ might require trying out many different possibilities -- each of which will consume privacy budget.


%% file: hyperparameter_selection.tex
This section introduces a method to jointly tune privacy parameters (e.g., noise scale) along with parameters related only to the utility of an algorithm (e.g., learning rate or batch size in stochastic gradient descent) -- while avoiding the $\perp$ output.

Algorithm~\ref{alg: parameter_ptr} takes a list of parameters as input, runs generalized PTR with each of the parameters, and returns the output with the best utility. We show that the privacy guarantee with respect to $\epsilon$ is independent of the number of $\phi$ that we try.  

Formally, let $\phi_1, ..., \phi_k$ be a set of hyper-parameters and $\tilde{\theta}_i \in\{\perp, \text{Range}(\cM)\}$ denotes the output of running generalized PTR on a private dataset $X$ with $\phi_i$. 
Let $X_{val}$ be a public validation set and $q(\tilde{\theta}_i)$ be the score of evaluating $\tilde{\theta}_i$ with $X_{val}$ (e.g., validation accuracy). The goal is to select a pair $(\tilde{\theta}_i$, $\phi_i)$ such that DP model $\tilde{\theta}_i$ maximizes the validation score.

The generalized PTR framework with privacy calibration is described in Algorithm~\ref{alg: parameter_ptr}. The privacy guarantee of Algorithm~\ref{alg: parameter_ptr} is an application of \citet{liu2019private}.

\begin{algorithm}[H]
	\caption{PTR with hyper-parameter selection}
	\label{alg: parameter_ptr}
	\begin{algorithmic}[1]
	   \STATE {\textbf{Input}:  Privacy budget per PTR algorithm ($\epsilon^*, \delta^*$), cut-off $T$, parameters $\phi_{1:k}$, flipping probability $\tau$ and validation score function $q(\cdot)$. } 
		\STATE {Initialize the set $S=\varnothing$.}
		\STATE{Draw $G$ from a geometric distribution $\cD_\tau$ and let $\hat{T}=\text{min}(T, G)$.}
		\FOR{i = 1 ,..., $\hat{T}$}
		\STATE{ pick a random $\phi_i$ from $\phi_{1:k}$.}
		\STATE{evaluate $\phi_i$: $(\tilde{\theta}_i, q(\tilde{\theta}_i))\gets$ Algorithm~\ref{alg:gen_ptr}($\phi_i, (\epsilon^*, \delta^*)$).}
		\STATE {$S \gets S \cup \{\tilde{\theta}_i, q(\tilde{\theta}_i)\}$.}
		\ENDFOR 
	\STATE{Output the highest scored candidate from $S$.}
	\end{algorithmic}
\end{algorithm}

\begin{theorem}[ Theorem 3.4 \citet{liu2019private} ]
Fix any $\tau \in [0, 1], \delta_2>0$ and let $T =\frac{1}{\tau} \log \frac{1}{\delta_2}$. If each oracle access to Algorithm~\ref{alg:gen_ptr} is $(\epsilon^*, \delta^*)$-DP, then
Algorithm~\ref{alg: parameter_ptr} is $(3\epsilon^* + 3\sqrt{2\delta^*}, \sqrt{2\delta^*} T +\delta_2 )$-DP.
\end{theorem}
The theorem implies that one can try a random number of $\phi$ while paying a constant $\epsilon$.
In practice, we can roughly set $\tau = \frac{1}{10k}$ so that the algorithm is likely to test all $k$ parameters. We emphasize that the privacy and the utility guarantee (stated in the appendix) is not our contribution. But the idea of applying generalized PTR to enforce a uniform DP guarantee over all choices of parameters with a data-dependent analysis is new, and in our opinion, significantly broadens the applicability to generic hyper-parameter tuning machinery from \citet{liu2019private}.

%% file: test_construction.tex
Classic PTR uses the Laplace mechanism to construct a differentially private upper bound of $\cD_{\beta}(X)$, the distance from input dataset $X$ to the closest dataset whose local sensitivity exceeds the proposed bound $\beta$. The tail bound of the Laplace distribution then ensures that if $\cD_{\beta}(X) = 0$ (i.e. if $\Delta_{LS}(X) > \beta$), then the output will be released with only a small probability $\delta$.

The following theorem shows that we could instead use a differentially private upper bound of the data-dependent DP $\epsilon_{\phi}(X)$ in order to test whether to run the mechanism $\cM_{\phi}$.

\begin{theorem}[Generalized PTR with private upper bound]\label{exp: upperbound}
Suppose we have a differentially private upper bound of $\epsilon_\phi(X)$ w.r.t. $\delta$ such that with probability at least $1-\delta'$, $\epsilon_{\phi }^P(X)>\epsilon_{\phi}(X)$. Further suppose we have an $(\hat{\epsilon}, \hat{\delta})$-DP test $\cT$ such that
\begin{align*}
    T(X) &= \begin{cases}
    1 & \text{ if } \epsilon_{\phi }^P(X) < \epsilon, \\
    0 & \text{ otherwise}.
    \end{cases}
\end{align*}

Then Algorithm~\ref{alg:gen_ptr} is $(\epsilon +\hat{\epsilon}, \delta +\hat{\delta} + \delta')$-DP. 
\end{theorem}


In Section~\ref{subsections:pate}, we demonstrate that one can upper bound the data-dependent DP through a modification of the smooth sensitivity framework applied on $\epsilon_\phi(X)$. Moreover, in Section~\ref{subsections:private_linear_regression} we provide a direct application of Theorem~\ref{exp: upperbound} with private linear regression by making use of the per-instance DP technique~\citep{wang2017per}.

The applications in Section~\ref{sections:applications} are illustrative of two distinct approaches to constructing the DP test for generalized PTR:

\begin{enumerate}
    \item Private sufficient statistics release (used in the private linear regression example of Section~\ref{subsections:private_linear_regression}) specifies the data-dependent DP as a function of the dataset and privately releases each data-dependent component. 
    \item The second approach (used in the PATE example of Section~\ref{subsections:pate}) uses the smooth sensitivity framework to privately release the data-dependent DP as a whole, and then construct a high-confidence test using the Gaussian mechanism. 
    
\end{enumerate}
These two approaches cover most of the scenarios arising in data-adaptive analysis. For example, in the appendix we demonstrate the merits of generalized PTR in handling data-adaptive private generalized linear models (GLMs)  using private sufficient statistics release. Moreover, sufficient statistics release together with our private hyper-parameter tuning (Algorithm~\ref{alg: parameter_ptr}) can be used to construct data-adaptive extensions of DP-PCA and Sparse-DP-ERM (see details in the future work section).

%% file: applications.tex
\section{Applications}
\label{sections:applications}
In this section, we put into action our approaches to construct the DP test and provide applications in private linear regression and PATE.


\subsection{Private Linear Regression}

\input{private_linear_regression}
\label{subsections:private_linear_regression}

\subsection{PATE}
\label{subsections:pate}
\input{pate}

%% file: private_linear_regression.tex
\begin{figure*}[t]
	\centering	
	\subfigure[Bike dataset]{
	\includegraphics[width=0.48\textwidth]{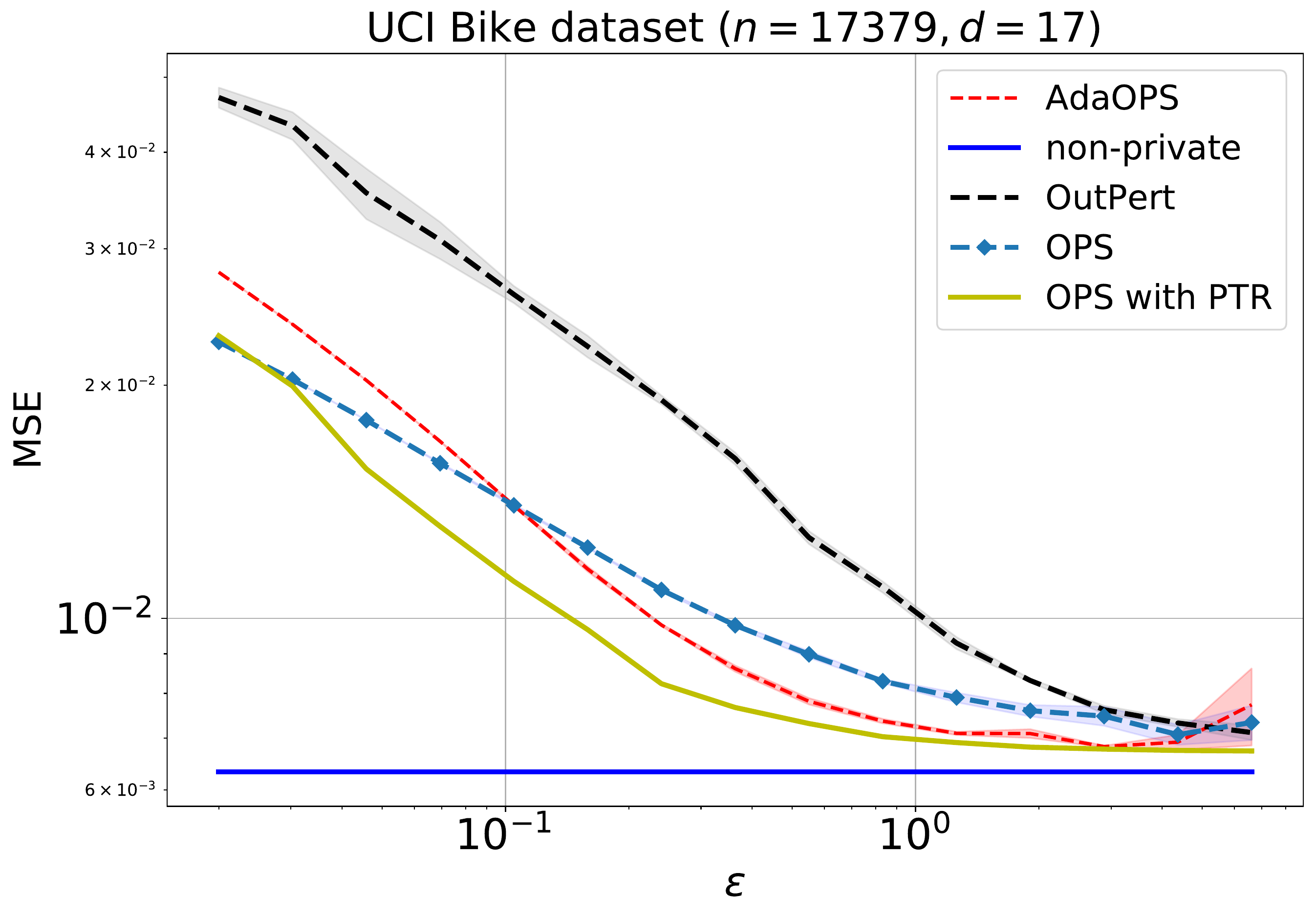}\label{fig:bike}}
	\subfigure[Elevators dataset]{
	\includegraphics[width=0.46\textwidth]{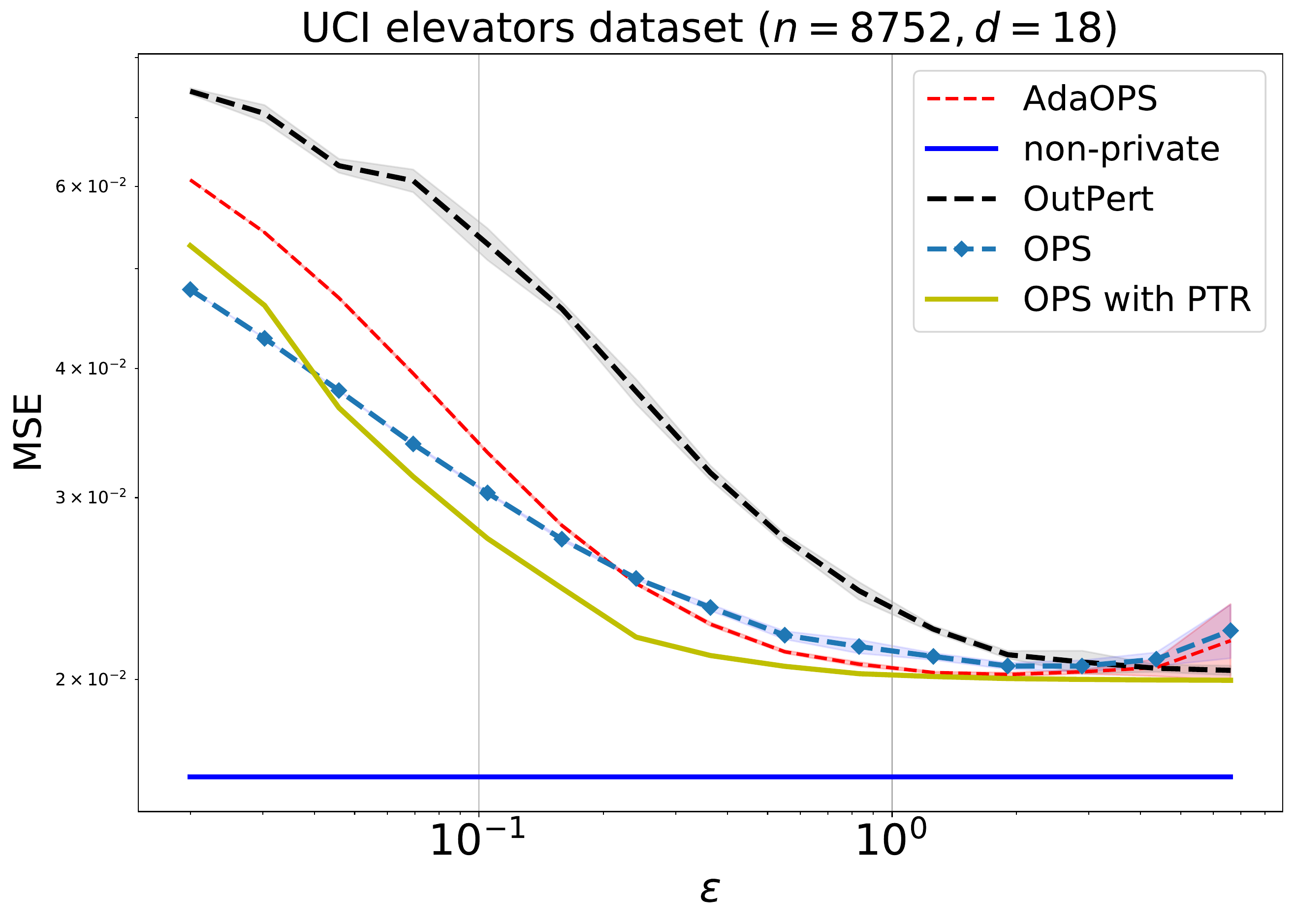}\label{fig:ele}}
	\caption{Differentially private linear regression algorithms on UCI datasets. $y$-axis reports the MSE error with confidence intervals. $\epsilon$ is evaluated with $\delta = 1e-6$. }
	\label{fig:selection}
\end{figure*}

\begin{theorem}[\citep{wang2017per}]\label{thm: per}
For input data $X \in \mathcal{X}$ and $Y \in \mathcal{Y}$, define the following:
\begin{itemize}
    \item $\lambda_{\min}(X)$ denotes the smallest eigenvalue of $X^TX$; 
    \item $||\theta_\lambda^*||$ is the magnitude of the solution $\theta_\lambda^* = (X^TX +\lambda I )^{-1}X^T Y$;
    \item and $L(X, \mathbf{y}):=||\cX||(||\cX||||\theta_\lambda^*||+||\cY||)$ is the local Lipschitz constant, denoted $L$ in brief.
\end{itemize}
For brevity, denote $\lambda^* = \lambda + \lambda_{\min}(X)$.
The algorithm used in Example~\ref{exp: posterior} with  parameter $\phi = (\lambda, \gamma)$ obeys $(\epsilon_{\phi}(Z), \delta)$ data-dependent DP for each dataset $ Z = (X, Y)$  with $\epsilon_{\phi}(Z)$ equal to
\[\sqrt{\frac{\gamma L^2 \log(2/\delta)}{\lambda^*}} + \frac{\gamma L^2}{2(\lambda^*  +||\cX||^2)}+ \frac{1 + \log(2/\delta)||\cX||^2}{2(\lambda^*)}.
\]
\end{theorem}

Notice that the data-dependent DP is a function of $(\lambda_{\min}, L, ||\theta_\lambda^*||, \lambda, \gamma)$, where $(\lambda_{\min}, L, ||\theta_\lambda^*||)$ are data-dependent quantities. One can apply the generalized PTR framework as in the following example.
\begin{example}[OPS with PTR] We demonstrate here how to apply generalized PTR to the one-posterior sample (OPS) algorithm, a differentially private mechanism which outputs one sample from the posterior distribution of a Bayesian model with bounded log-likelihood.
\label{examp:ops}
\begin{itemize}
    \item Propose $\phi=(\lambda, \gamma)$.    
    \item Based on $(\lambda, \gamma)$, differentially privately release $\lambda_{min}, ||\theta_\lambda^*||, L$ 
    with privacy budget $(\epsilon, \delta/2)$.    
    \item Condition on a high probability event (with probability at least $1-\delta/2$) of $\lambda_{min}, ||\theta_\lambda^*||, L$, test if $\blue{\epsilon_{\phi}^P(X)}$  is smaller than the predefined privacy budget $(\hat{\epsilon}, \hat{\delta})$, where $\epsilon_\phi^P(X)$ denotes the sanitized data-dependent DP.   
    \item Based on the outcome of the test, decide whether to release $\theta \propto e^{-\frac{\gamma}{2}||Y-X\theta||^2 + \lambda||\theta||^2}$.   
\end{itemize}
\begin{theorem}
The algorithm outlined in Example~\ref{examp:ops} satisfies $(\epsilon+ \hat{\epsilon}, \delta + \hat{\delta})$-DP. 
\end{theorem}
\end{example}

The main idea of the above algorithm boils down to  privately releasing all data-dependent quantities in data-dependent DP, constructing high-probability confidence intervals of these quantities, and then deciding whether to run the mechanism $\cM$ with the proposed parameters. We defer the details of the privacy calibration of data-dependent quantities to the appendix. 

One may ask why we cannot directly tune privacy parameters ($\lambda, \gamma$) based on the sanitized data-dependent DP. This is because, in many scenarios, data-dependent quantities depend on the choice of privacy parameters, e.g., $||\theta_\lambda^*||$ is a complicated function of $\lambda$. Thus, the optimization on $\lambda$ becomes a circular problem --- to solve $\lambda$, we need to sanitize $||\theta_
\lambda^*||$, which needs to choose a $\lambda$ to begin with. Alternatively, generalized PTR provides a clear and flexible framework to test the validity of privacy parameters adapted to the dataset. 

\begin{remark}
The above ``circular'' issue is even more serious for generalized linear models (GLMs) beyond linear regression. The data-dependent DP there involves a local strong-convexity parameter, a complex function of the regularizer $\lambda$  and we only have zeroth-order access to. In the appendix, we demonstrate how to apply generalized PTR to provide a generic solution to a family of private GLMs where the link function satisfies a self-concordance assumption.
\end{remark}

We next apply Algorithm~\ref{alg: parameter_ptr} for Example~\ref{examp:ops}  with UCI regression datasets. Standard z-scoring is applied and each data point is normalize with a Euclidean norm of 1.  We consider $(60\%,10\%,30\%)$ splits for training, validation and testing test.

\textbf{Baselines}
\begin{itemize}
    \item Output Perturbation (Outpert)~\citep{chaudhuri2011differentially}: $\theta=(X^TX + \lambda I)^{-1}X^T\bf{y}$. Release $\hat{\theta} = \theta + \bf{b}$ with an appropriate $\lambda$, where $\bf{b}$ is a Gaussian random vector.
    \item Posterior sampling (OPS). Sample $\hat{\theta}\sim P(\theta)\propto e^{-\gamma(F(\theta)+ 0.5 \lambda ||\theta||^2)}$ with parameters $\gamma, \lambda$.
    \item Adaptive posterior sampling (AdaOPS)~\citep{wang2018revisiting}. Run OPS with $(\lambda, \gamma)$ chosen adaptively according to the dataset.
\end{itemize}
Outpert and OPS serve as two non-adaptive baselines. In particular, we consider OPS-Balanced~\citep{wang2018revisiting}, which chooses $\lambda$ to minimize a data-independent upper bound of empirical risk and dominates other OPS variants. 
AdaOPS is one state-of-the-art algorithm for adaptive private regression, which  automatically chooses $\lambda$ by minimizing an upper bound of the data-dependent empirical risk.

We implement OPS-PTR as follows: 
propose a list of $\lambda$ through grid search (we choose $k=30$ and $\lambda$ ranges from $[2.5, 2.5^{10}]$ on a logarithmic scale); instantiate Algorithm~\ref{alg: parameter_ptr} with $\tau =0.1k$, $T=\frac{1}{\tau} \log(1/\delta_2)$ and  $\delta_2 = 1/2 \delta$;
calibrate $\gamma$ to meet the privacy requirement for each $\lambda$. sample $\hat{\theta}$ using $(\lambda,\gamma)$ and return the one with the best validation accuracy. Notice that we use a ``no $\perp$'' variant of Algorithm~\ref{alg:gen_ptr} as the calibration of $\gamma$ is clear given a fixed $\lambda$ and privacy budget (see more details in the appendix). We can propose various combinations of $(\lambda, \gamma)$ for more general applications.

 Figure~\ref{fig:selection}  demonstrates how the MSE error of the linear regression algorithms varies with the privacy budget $\epsilon$. OutPert suffers from the large global sensitivity of output $\theta$. OPS performs well but does not benefit from the data-dependent quantities. AdaOPS is able to adaptively choose $(\lambda, \gamma)$ based on the dataset, but suffers from the estimation error of the data-dependent empirical risk. On the other hand, OPS-PTR selects a  $(\lambda, \gamma)$ pair that minimizes the empirical error on the validation set directly, and the privacy parameter $\gamma$ adapts to the dataset thus achieving the best result.

%% file: pate.tex
In this section, we apply the generalized PTR framework to solve an open problem from the Private Aggregation of Teacher Ensembles (PATE) \citep{papernot2017, papernot2018scalable} --- privately publishing the entire model through privately releasing data-dependent DP losses. Our algorithm makes use of the smooth sensitivity framework ~\citep{nissim2007smooth} and the Gaussian mechanism to construct a high-probability test of the data-dependent DP. The one-dimensional statistical nature of data-dependent DP enables efficient computations under the smooth sensitivity framework. Thus, this approach is generally applicable for other private data-adaptive analysis beyond PATE.


PATE  is a knowledge transfer framework for model-agnostic private learning. In this framework,  an ensemble of teacher models is trained on the disjoint private data and uses the teachers' aggregated consensus answers to supervise the training of a ``student'' model agnostic to the underlying machine-learning algorithms. By publishing only the aggregated answers and by the careful analysis of the ``consensus'', PATE has become a practical technique in recent private model training.

The tight privacy guarantee of PATE heavily relies on a delicate data-dependent DP analysis, for which the authors of PATE use the smooth sensitivity framework  to privately publish the data-dependent privacy cost. However, it remains an open problem to show that the released model is DP under  data-dependent analysis. Our generalized PTR resolves this gap by  carefully testing a private upper bound of the data-dependent privacy cost. Our algorithm is fully described in Algorithm~\ref{alg: pate_ptr}, where the modification over the original PATE framework is highlighted in blue.

 
Algorithm~\ref{alg: pate_ptr} takes the input of privacy budget $(\epsilon',\hat{\epsilon}, \delta)$, unlabeled public data $x_{1:T}$ and $K$ teachers' predictions on these data.  The parameter $\epsilon$ denotes the privacy cost of publishing the data-dependent DP and $\epsilon'$ is the predefined privacy budget for testing. $n_j(x_i)$ denotes the the number of teachers that agree on label $j$ for $x_i$ and $C$ denotes the number of classes. The goal is to privately release a list of plurality outcomes --- $\text{argmax}_{j\in[C]} n_j(x_i)$ for $i \in[T]$ --- and use these outcomes to supervise the training of a ``student'' model in the public domain. The parameter $\sigma_1$ denotes the noise scale for the vote count.

In their privacy analysis,  \citet{papernot2018scalable} compute the data-dependent $\rdp_{\sigma_1}(\alpha, X)$ of labeling the entire group of student queries.  $\rdp_{\sigma_1}(\alpha, X)$ can be orders of magnitude smaller than its data-independent version if there is a strong agreement among teachers. Note that $\rdp_{\sigma_1}(\alpha, X)$  is a function of the RDP order $\alpha$ and the dataset $X$, analogous to our  Definition~\ref{def:data_dep_dp}  but subject to RDP~\citep{mironov2017renyi}.

\begin{theorem}[\citep{papernot2018scalable}]\label{thm: dep_gau}
If the top three vote counts of $x_i$ are $n_1>n_2>n_3$ and $n_1 -n_2, n_2-n_3\gg \sigma_1$, then the data-dependent RDP of releasing $\text{argmax}_j \{n_j +\cN(0, \sigma_1^2)\}$
satisfies $(\alpha, \exp\{-2\alpha/{\sigma_1^2}\}/\alpha)$-RDP and the data-independent RDP (using the Gaussian mechanism) satisfies $(\alpha, \frac{\alpha}{\sigma_1^2})$-RDP. 
\end{theorem}

\begin{algorithm}[H]
	\caption{PATE with generalized PTR}
	\label{alg: pate_ptr}
	\begin{algorithmic}[1]
		\STATE {\textbf{Input}: Unlabeled public data $x_{1:T}$, aggregated teachers prediction $n(\cdot)$, privacy parameter $\hat{\epsilon},\epsilon', \delta$, noisy parameter $\sigma_1.$}
		\STATE{Set  $\alpha =\frac{2\log(2/\delta)}{\hat{\epsilon}}+1 $, $\sigma_s = \sigma_2 = \sqrt{\frac{3\alpha + 2}{\hat{\epsilon}}}, \delta_2 = \delta/2, $ smoothness parameter $ \beta = \frac{0.2}{\alpha}$.}
		\STATE{Compute noisy labels: ${y_i}^p \gets \text{argmax}_{j\in [C]}\{n_j(x_i)+ \cN(0, \sigma_1^2)\} $ for all $i \in[1:T]$.}
		\STATE{$\rdp_{\sigma_1}(\alpha, X)\gets$ data-dependent RDP at the $\alpha$-th order.}
		\STATE {$SS_\beta(X) \gets$ the smooth sensitivity  of $\rdp_{\sigma_1}^{\text{upper}} (\alpha,X)$.}
		\STATE \blue{Privately release $\mu:=  \log(SS_\beta(X)) + \beta \cdot \cN(0, \sigma_2^2)+\sqrt{2\log(2/\delta_2)}\cdot \sigma_2 \cdot \beta$}
		\STATE{\blue{ $\rdp_{\sigma_1}^{\text{upper}}(\alpha)\gets $ an upper bound of data-dependent RDP through Lemma~\ref{lem: upperbound}}}.
		\STATE{\blue{$\epsilon_{\sigma_1} \gets$ DP guarantee converted from $\rdp_{\sigma_1}^{\text{upper}}(\alpha).$ }}
		\STATE{\blue{If $\epsilon'\geq \epsilon_{\sigma_1}$ \textbf{return} a student model trained using $(x_{1:T}; y_{1:T}^p)$}}.
		\STATE{\blue{\text{Else return} $\perp$.}}
	\end{algorithmic}
\end{algorithm}

However, $\rdp_{\sigma_1}(\alpha, X)$ is data-dependent and thus cannot be revealed. The authors therefore privately publish the data-dependent RDP using the smooth sensitivity framework~\citep{nissim2007smooth}.
The smooth sensitivity calculates a smooth upper bound on the local sensitivity of $\rdp_{\sigma_1}(\alpha, X)$, denoted as $SS_\beta(X)$, such that $SS_\beta (X) \leq e^\beta SS_\beta(X')$ for any neighboring dataset $X$ and $X'$. By adding Gaussian noise scaled by the smooth sensitivity (i.e., release $\epsilon_{\sigma_1}(\alpha, X)+ SS_\beta(X)\cdot \cN(0, \sigma_s^2)$), the privacy cost is safely published.

Unlike most noise-adding mechanisms, the standard deviation $\sigma_s$ cannot be published since $SS_\beta(X)$ is a data-dependent quantity. Moreover, this approach fails to provide a valid privacy guarantee of the noisy labels obtained through the PATE algorithm, as the published privacy cost could be smaller than the real privacy cost.
Our solution in Algorithm~\ref{alg: pate_ptr} looks like the following:
\begin{itemize}
    \item Privately release an upper bound of the smooth sensitivity $SS_\beta(X)$ with $e^{\mu}$. 
    \item Conditioned on a high-probability event of $e^{\mu}$, publish the data-dependent RDP with $\rdp_{\sigma_1}^{\text{upper}}(\alpha)$.
    \item Convert $\rdp_{\sigma_1}^{\text{upper}}(\alpha)$ back to the standard DP guarantee using RDP to DP conversion at $\delta/2$.  
    \item Test if the converted DP is above the predefined budget $\epsilon'$.
\end{itemize}
The following lemma states that $\rdp_{\sigma_1}^{\text{upper}}(\alpha)$ is a valid upper bound of the data-dependent RDP.

\begin{figure*}[t]
	\centering	
	\subfigure[High consensus and strong data-dependent DP ]{
	\includegraphics[width=0.46\textwidth]{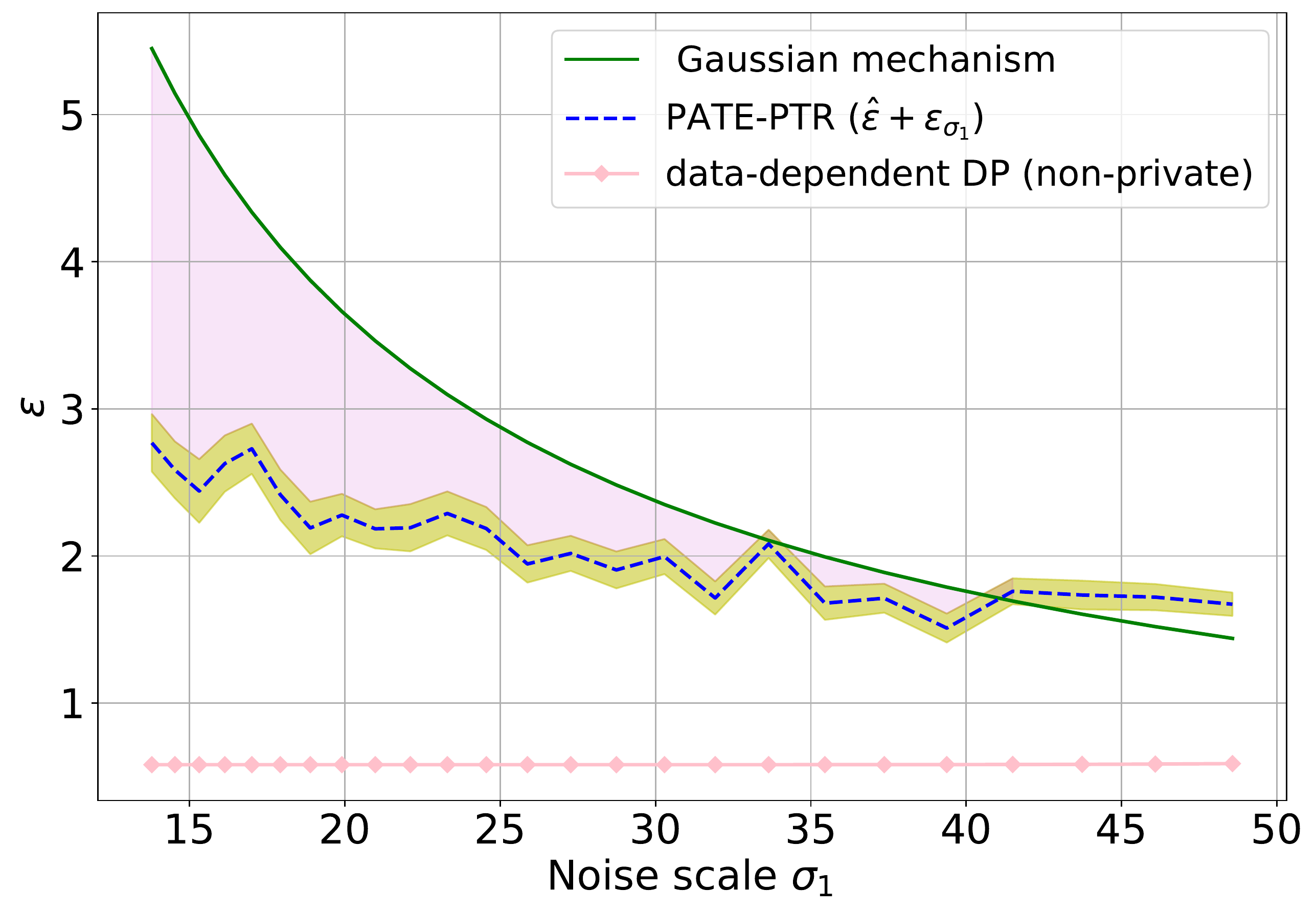}\label{fig:high}}
	\subfigure[Low consensus and low data-dependent DP]{
	\includegraphics[width=0.46\textwidth]{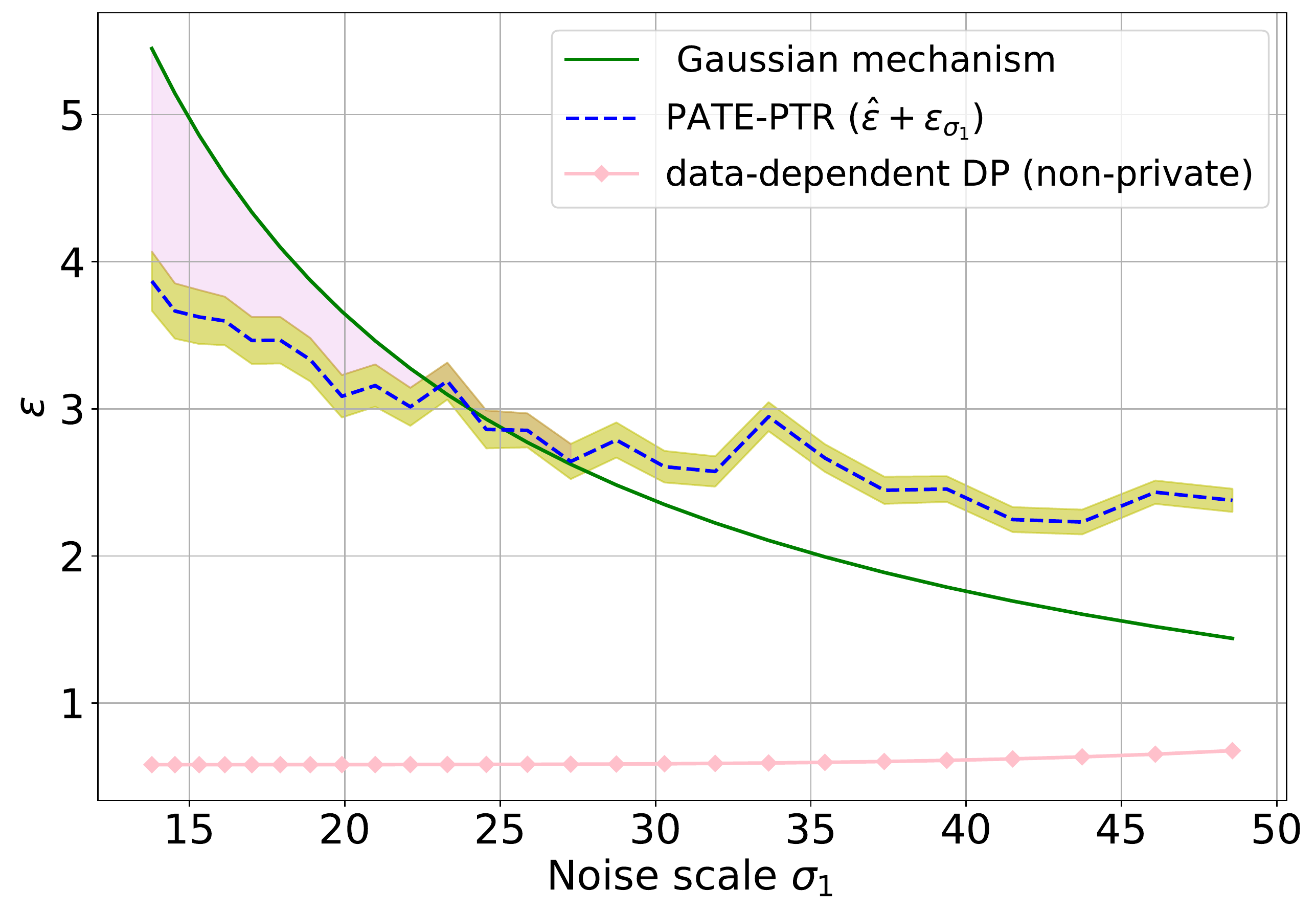}\label{fig:low}} 
	\caption{Privacy and utility tradeoffs with PATE. 
	When $\sigma_1$ is aligned, three algorithms provide the same utility. $y$-axis plots the privacy cost of labeling $T=200$ public data with $\delta = 10^{-5}$.
   The left figure considers the high-consensus case, where the data-adaptive analysis is preferred.} 
	\label{fig: exp_pate}
\end{figure*}

\begin{lemma}[Private upper bound of data-dependent  RDP]\label{lem: upperbound}
We are given a RDP function $\rdp(\alpha, X)$ and a $\beta$-smooth sensitivity bound $SS(\cdot)$ of $\rdp(\alpha, X)$. Let $\mu$ (defined in Algorithm~\ref{alg: pate_ptr}) denote the private release of $\log(SS_\beta(X))$. Let the $(\beta, \sigma_s, \sigma_2)$-GNSS mechanism be 
\[\scriptstyle
\rdp^{\text{upper}}(\alpha):=\rdp(\alpha, X) + SS_\beta(X) \cdot \cN(0, \sigma_s^2) + \sigma_s \sqrt{2\log(\frac{2}{\delta_2}) } e^{\mu} \]
	 Then, the release of $\rdp^{\text{upper}}(X)$ satisfies $(\alpha, \frac{3\alpha +2}{2\sigma_s^2})$-RDP for all $1<\alpha < \frac{1}{2\beta}$; w.p. at least $1-\delta_2$, $\rdp^{\text{upper}}(\alpha)$ is an upper bound of $\rdp(\alpha, X)$.
\end{lemma}

The proof (deferred to the appendix) makes use of the facts that: (1) the log of $SS_\beta(X)$ has a bounded global sensitivity $\beta$ through the definition of smooth sensitivity; (2) releasing $\rdp_{\sigma_1}(\alpha, X)+ SS_\beta(X)\cdot \cN(0, \sigma_s^2)$ is $(\alpha, \frac{\alpha+1}{\sigma_s^2})$-RDP (Theorem 23 from \citet{papernot2018scalable}).



Now, we are ready to state the privacy guarantee of Algorithm~\ref{alg: pate_ptr}.

\begin{theorem}\label{thm: pate_ptr}
Algorithm~\ref{alg: pate_ptr} satisfies $(\epsilon'+\hat{\epsilon}, \delta)$-DP.
\end{theorem}
In the proof, the choice of $\alpha$ ensures that the cost of the $\delta/2$ contribution (used in the RDP-to-DP conversion) is roughly $\hat{\epsilon}/2$. Then  the release of $\rdp_{\sigma_1}^{\text{upper}}(\alpha)$ with  $\sigma_s =\sqrt{\frac{2+3\alpha}{\hat{\epsilon}}}$ accounts for another cost of $(\epsilon/2, \delta/2)$-DP.



\textbf{Empirical results.} We next empirically evaluate Algorithm~\ref{alg: pate_ptr} (PATE-PTR) on the MNIST dataset. Following the experimental setup from \citet{papernot2018scalable}, we consider the training set to be the private domain, and the testing set is used as the public domain.  We first partition the training set into $400$ disjoint sets and $400$ teacher models, each trained individually. Then we select $T=200$ unlabeled data from the public domain, with the goal of privately labeling them. 
To illustrate the behaviors of algorithms under various data distributions, we consider two settings of unlabeled data, high-consensus and low-consensus. In the low-consensus setting, we choose $T$ unlabeled data such that there is no high agreement among teachers, so the advantage of data-adaptive analysis is diminished. We provide further details on the distribution of these two settings in the appendix.

\textbf{Baselines.}
We consider the Gaussian mechanism as a data-independent baseline, where the privacy guarantee is valid but does not take advantage of the properties of the dataset. The data-dependent DP (~\citet{papernot2018scalable}) serves as a non-private baseline, which requires further sanitation.  Note that these two baselines provide different privacy analyses of the same algorithm (see Theorem~\ref{thm: dep_gau}).

Figure~\ref{fig: exp_pate} plots privacy-utility tradeoffs between the three approaches by varying the noise scale $\sigma_1$. The purple region denotes a set of privacy budget choices ($\hat{\epsilon}+\epsilon'$ used in Algorithm~\ref{alg: pate_ptr}) such that the utility of the three algorithms is aligned under the same $\sigma_1$. In more detail, the purple region is lower-bounded by $\hat{\epsilon} +\epsilon_{\sigma_1}$.
We first fix $\sigma_s=\sigma_2=15$ such that $\hat{\epsilon}$ is fixed. Then we empirically calculate the average of $\epsilon_{\sigma_1}$ (the private upper bound of the data-dependent DP) over $10$ trials. Running Algorithm~\ref{alg: pate_ptr} with any choice of $\hat{\epsilon}+\epsilon'$ chosen from the purple region implies  $\epsilon'>\epsilon_{\sigma_1}$. Therefore, PATE-PTR will output the same noisy labels (with high probability) as the two baselines.

\textbf{Observation} 
As $\sigma_1$ increases, the privacy loss of the Gaussian mechanism decreases, while the data-dependent DP curve does not change much. This is because the data-dependent DP of each query is a complex function of both the noise scale and the data and does not monotonically decrease when $\sigma_1$ increases (see more details in the appendix). However, the data-dependent DP still dominates the Gaussian mechanism for a wide range of $\sigma_1$. Moreover, PATE-PTR nicely interpolates between the data-independent DP guarantee and the non-private data-adaptive DP guarantee.  In the low-consensus case, the gap between the data-dependent DP and the DP guarantee of the Gaussian mechanism unsurprisingly decreases. Meanwhile, PATE-PTR (the purple region) performs well when the noise scale is small but deteriorates when the data-independent approach proves more advantageous. 
This example demonstrates that using PTR as a post-processing step to convert the data-dependent DP to standard DP is effective when the data-adaptive approach dominates others.


%% file: conclusion.tex
\section{Limitations and Future Work}
One weakness of generalized PTR is that it requires a case-specific privacy analysis. Have we simply exchanged the problem of designing a data-adaptive DP algorithm with the problem of analyzing the data-dependent privacy loss? We argue that this limitation is inherited from classic PTR. In situations where classic PTR is not applicable, we've outlined several approaches to constructing the DP test for our framework (see Sections~\ref{subsections:test_construction} and~\ref{section:applications}).

Furthermore, the data-dependent privacy loss is often more straightforward to compute than local sensitivity, and often exists in intermediate steps of classic DP analysis already. Most DP analysis involves providing a high-probability tail bound of the privacy loss random variable. If we stop before taking the max over the input dataset, then we get a data-dependent DP loss right away (as in Example~\ref{examp:lap_mech}).

There are several exciting directions for applying generalized PTR to more problems. Sufficient statistics release and our private hyperparameter tuning (Algorithm~\ref{alg: parameter_ptr}) can be used to construct data-adaptive extensions of DP-PCA \citep{dwork2014analyze} and Sparse-DP-ERM \citep{kifer2012private}. For DP-PCA we could use our Algorithm~\ref{alg: parameter_ptr} to tune the variance of the noise added to the spectral gap; for Sparse-DP-ERM we would test the restricted strong convexity parameter (RSC), i.e. not adding additional regularization if the RSC is already large.

\section{Conclusion}

Generalized PTR extends the classic ``Propose-Test-Release'' framework to a more general setting by testing the data-dependent privacy loss of an input dataset, rather than its local sensitivity. In this paper we've provided several examples -- private linear regression with hyperparameter selection and PATE -- to illustrate how generalized PTR can enhance DP algorithm design via a data-adaptive approach.

%% file: appendix.tex

\tableofcontents
\section{Omitted examples in the main body}

In this appendix, we provide more examples to demonstrate the merits of generalized PTR. We focus on a simple example of post-processed Laplace mechanism in Section~\ref{sec:binary_vote} and then an example on differentially private learning of generalized linear models in Section~\ref{sec:gen_ptr}. In both cases, we observe that generalized PTR provides data-adaptive algorithms with formal DP guarantees, that are simple, effective and not previously proposed in the literature (to the best of our knowledge).

\subsection{Limits of the classic PTR in private binary voting}\label{sec:binary_vote}

The following example demonstrates that  classic PTR  does not capture sufficient data-dependent quantities
even when the local sensitivity exists and can be efficiently tested.
\begin{example}\label{exp: binary_vote}
Consider a binary class voting problem: $n$ users vote for a binary class $\{0, 1\}$ and the goal is to output the class that is supported by the majority. Let $n_i$ denote the number of people who vote for the class $i$. We consider the report-noisy-max mechanism: 
\begin{align*}
\cM(X): \text{argmax}_{i \in [0,1]} n_i(X)+ Lap(b),
\end{align*}
where $b=1/\epsilon$ denotes the scale of Laplace noise. 

\end{example}
In the example, we will (1) demonstrate the merit of data-dependent DP; and 
(2) empirically compare classic PTR with generalized PTR.

We first explicitly state the data-dependent DP.
\begin{theorem}\label{thm: binary_vote}
The data-dependent DP of the above example is 
\[\epsilon(X):= \max_{X'}\{ |\log\frac{p}{p'}|, |\log\frac{1-p}{1-p'}|\},\] where $p:=\Pr[n_0(X) + Lap(1/\epsilon)> n_1(X) +Lap(1/\epsilon)]$ and $p':= \Pr[n_0(X') + Lap(1/\epsilon)> n_1(X')+Lap(1/\epsilon)]$. There are four possible neighboring datasets $X': n_0(X')=\max(n_0(X)\pm 1,0), n_1(X')=n_1(X)$ or  $n_0(X')=n_0(X), n_1(X')=\max(n_1(X)\pm 1,0)$.
\end{theorem}
In Figure~\ref{fig:gap}, we empirically compare the above data-dependent DP with the Laplace mechanism by varying the gap between the two vote counts $|n_0(X)-n_1(X)|$.  The noise scale is fixed to $\epsilon=10$.
The data-dependent DP substantially improves over the standard DP if the gap is large. However, the data-dependent DP is a function of the dataset. We next demonstrate how to apply generalized PTR to exploit the data-dependent DP.

Notice that the probability $n_0(X) + Lap(1/\epsilon)> n_1(X) +Lap(1/\epsilon)$  is equal to the probability that a random variable $Z:= X-Y$ exceeds $\epsilon(n_1(X)- n_0(X))$, where $X, Y$ are two independent $\text{Lap}(1)$ distributions. We can compute the pdf of $Z$ through the convolution of two Laplace distributions, which implies $f_{X-Y}(z) = \dfrac{1+|z|}{4e^{|z|}}$.  Let $t$ denote the difference between $n_1(X)$ and $n_0(X)$, i.e., $t=n_1(X)-n_0(X)$. Then we have
\[  p = \pr[Z>\epsilon \cdot t] = \frac{2+\epsilon \cdot t}{4 \exp(\epsilon \cdot t)}\]
Similarly, $p' =\dfrac{2+\epsilon \cdot (t+ \ell)}{4 \exp(\epsilon \cdot (t+\ell))} $, where $\ell\in[-1,1]$ denotes adding or removing one data point to construct the neighboring dataset $X'$.
Therefore, we can upper bound $\log(p/{p'})$ by
\begin{align*}
\log\frac{p}{p'} &= \frac{2 +\epsilon \cdot t}{4\exp(\epsilon \cdot t)}\cdot \frac{4\exp(\epsilon(t+\ell))}{2+\epsilon\cdot ( t+\ell)} \\ &\leq \epsilon \cdot \log\bigg(\frac{2+\epsilon t}{2+\epsilon(t+1)}\bigg) \\ &=\epsilon \log\bigg(1- \frac{\epsilon}{2+\epsilon(t+1)} \bigg)
\end{align*}

Then we can apply  generalized PTR by privately lower-bounding $t$.

On the other hand, the local sensitivity $\Delta_{LS}(X)$ of this noise-adding mechanism is $0$ if $t>1$. Specifically,  if the gap is larger than one, adding or removing one user will not change the result. To apply  classic PTR, we let  $\gamma(X)$ denote the distance to the nearest dataset $X^{''}$ such that $\Delta_{LS}>0$ and test if $\gamma(X)+\text{Lap}(1/\epsilon)> \frac{\log(1/\delta)}{\epsilon}$.
 Notice in this example that $\gamma(X)=\max(t-1, 0)$ can be computed efficiently.  
We provide the detailed implementation of these approaches.
\begin{enumerate}
    \item Gen PTR: lower bound $t$ with $t^p = t - \frac{log(1/\delta)}{\tilde{\epsilon}} + \text{Lap}(1/\tilde{\epsilon})$.
    Calculate an upper bound of data-dependent DP $\epsilon^p$ using Theorem~\ref{thm: binary_vote} with $t^p$. The algorithm then tests if $\epsilon^p$ is within an predefined privacy budget $\epsilon'$. If the test passes, the algorithm
    returns $\text{argmax}_{i \in [0,1]} n_i(X)+ Lap(1/\epsilon)$ satisfies $(\tilde{\epsilon}+\epsilon', \delta)$-DP.
    \item classic PTR: lower bound $t$ with $t^p = t - \frac{log(1/\delta)}{\tilde{\epsilon}} + \text{Lap}(1/\tilde{\epsilon})$. If $t^p>1$,  classic PTR outputs the ground-truth result else returns a random class. This algorithm satisfies $(\tilde{\epsilon}, \delta)$-DP.
    \item Laplace mechanism. $\cM(X): \text{argmax}_{i \in [0,1]} n_i(X)+ Lap(1/\epsilon)$. $\cM$ is $(\epsilon, \delta)$-DP.
\end{enumerate}
\begin{figure*}[t]
	\centering		\subfigure[data-dependent DP vs Laplace mechanism ]{
	\includegraphics[width=0.47\textwidth]{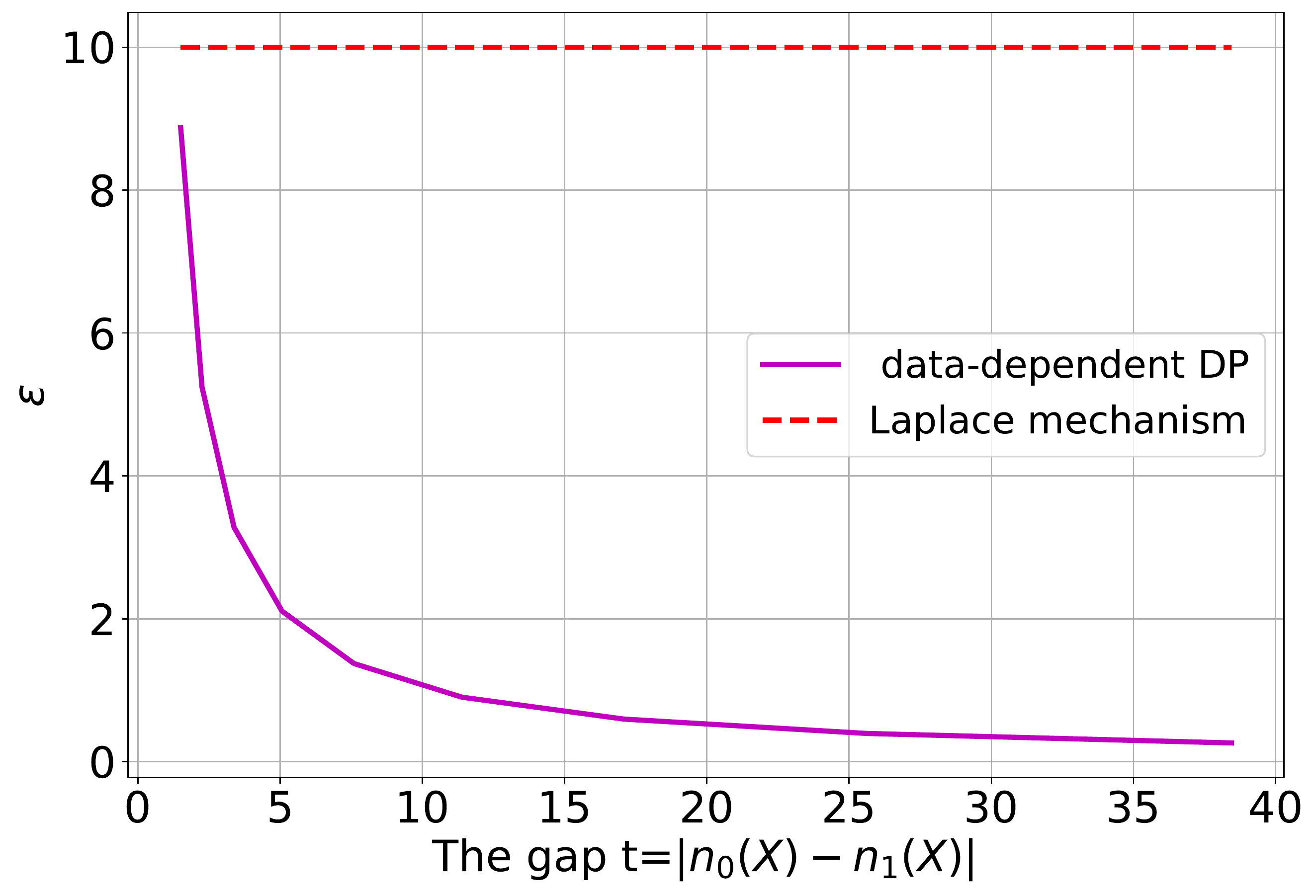}\label{fig:gap}}	\subfigure[ Privacy-utility tradeoff between three approaches. ]{
	\includegraphics[width=0.50\textwidth]{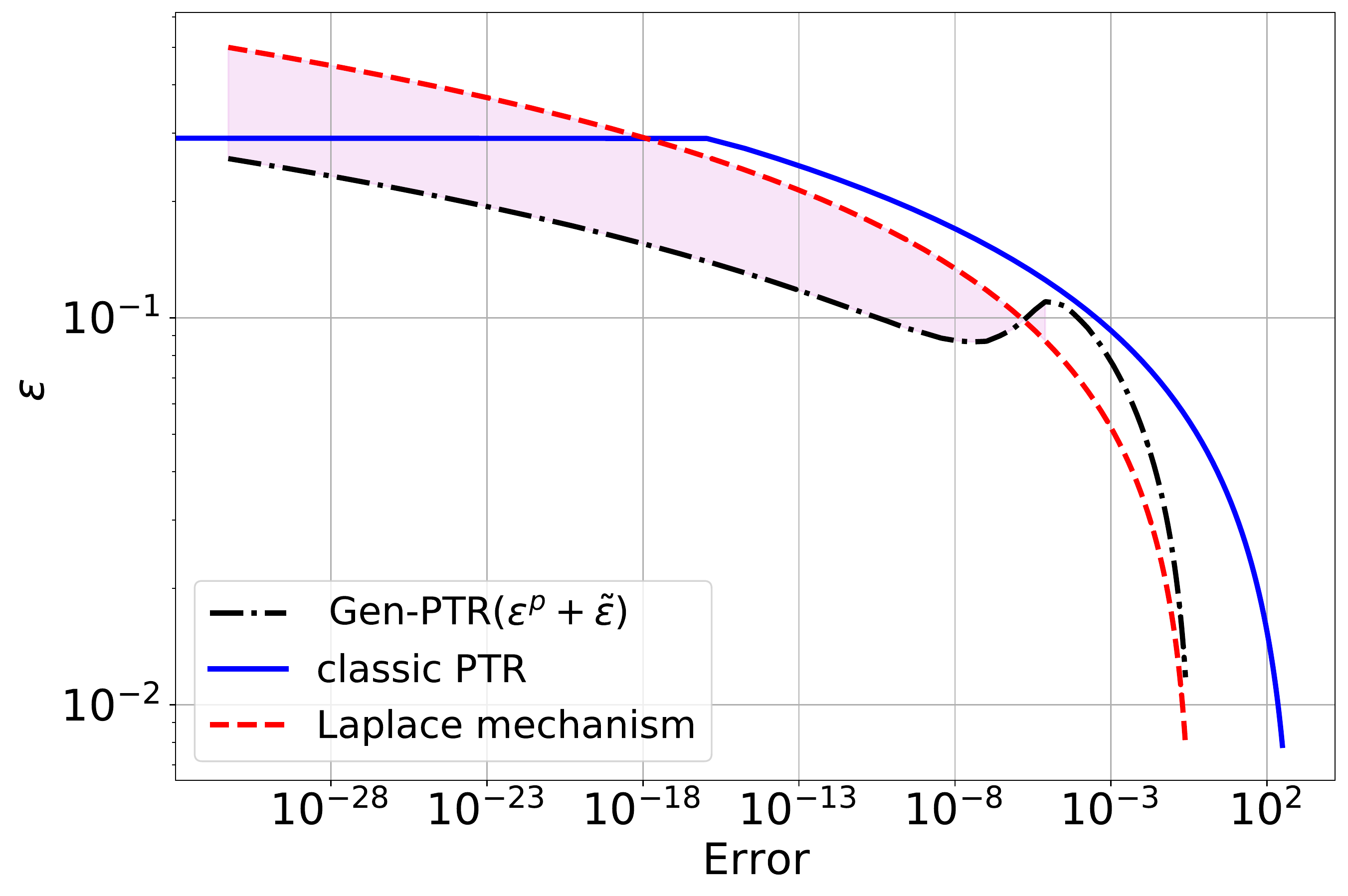}\label{fig:comp}}
	\caption{In Figure~\ref{fig:gap}, we compare the privacy guarantee by varying the gap. In Figure~\ref{fig:comp} We fix $t=n_0(X)-n_1(X)=100$ and compare privacy cost when the accuracy is aligned. 
	 Gen-PTR with any choice of privacy budget ($\tilde{\epsilon}+\epsilon'$) chosen from the purple region would achieve the same utility as Laplace mechanism but with a smaller privacy cost. The curve of Gen-PTR is always below than that of the classic PTR, which implies that Gen-PTR can result a tighter privacy analysis when the utility is aligned. } 
\end{figure*}

We argue that though the Gen-PTR and the classic PTR are similar in privately lower-bounding the data-dependent quantity $t$, the latter does not capture sufficient information for data-adaptive analysis. That is to say, only testing the local sensitivity restricts us from learning helpful information to amplify the privacy guarantee if the test fails. In contrast, our generalized PTR, where privacy parameters and the  local sensitivity parameterize the data-dependent DP, can handle those failure cases nicely.

To confirm this conjecture, Figure~\ref{fig:comp} plots a privacy-utility trade-off curve between these three approaches.
We consider a voting example with $n_0(X)=n_1(X)+100$ and $t=100$, chosen such that the data-adaptive analysis is favorable. 

In Figure~\ref{fig:comp}, we vary the noise scale $b=1/\epsilon$ between $[0, 0.5]$. For each choice of $b$, we plot the privacy guarantee of three algorithms 
 when the error rate is aligned. For Gen-PTR, we set $\tilde{\epsilon} = \frac{1}{2b}$ and empirically calculate $\epsilon^p$ over $100000$ trials. 

In the plot, when $\epsilon \ll \frac{\log(1/\delta)}{t}$,  the classic PTR is even worse than the Laplace mechanism. This is because the classic PTR is likely to return $\perp$ while the Laplace mechanism returns $\text{argmax}_{i\in[0,1]} n_i(X) + \text{Lap}(1/\epsilon)$, which contains more useful information. 
Compared to the Laplace mechanism, Gen-PTR requires an extra privacy allocation $\tilde{\epsilon}$ to release the gap $t$. However, it still achieves an overall smaller privacy cost when the error rate $\leq 10^{-5}$ (the purple region).  
Meanwhile,  Gen-PTR dominates the classic PTR (i.e., the dashed black curve is always below the blue curve). Note that the classic PTR and the Gen-PTR utilize the gap information differently: the classic PTR outputs $\perp$ if the gap is not sufficiently large, while the Gen-PTR encodes the gap into the data-dependent DP function and tests the data-dependent DP in the end. This empirical result suggests that testing the local sensitivity can be loosely compared to testing the data-dependent DP. Thus, Gen-PTR could provide a better privacy-utility trade-off.

\subsection{Self-concordant generalized linear model (GLM)}\label{sec:glm}

In this section, we demonstrate the effectiveness and flexibility of generalized PTR in handling a family of GLMs where the link function satisfies a self-concordance assumption.  This section is organized as follows:
\begin{itemize}
    \item Introduce a family of GLMs with the self-concordance property.
    \item Introduce a general output perturbation algorithm for private GLMs.
    \item Analyze the data-dependent DP of GLMs with the self-concordance property.
    \item Provide an example of applying our generalized PTR framework to logistic regression.
\end{itemize}

Consider the empirical risk minimization problem of the generalized linear model
\[
\theta^* = \argmin_\theta \sum_{i=1^n} l_i(\theta)+r(\theta),
\]

where  $l: \R\times \R \rightarrow \R$ belongs to a family of convex GLMs: $l_i (\theta) =  l(y, x_i^T\theta)$. Let $r: \R^d \rightarrow \R$ be a regularization function.

We now define the self-concordance property.
\begin{definition}[Generalized self-concordance {\citep{bach2010self}}]
	A convex and three-times differentiable function $f: \Theta \rightarrow \R$ is $R$-generalized-self-concordant on an open nonempty convex set $\Theta^*\subset \Theta$ with respect to norm $\|\cdot\|$ if for all $u\in \Theta^*$ and all $v\in \R^d$,
	$$
	\nabla^3 f(u)  [v,v,v]  \leq 2 R\|v\|(\nabla^2 f(u)[v,v]).
	$$
\end{definition}

The closer R is to 0, the ``nicer'' --- more self-concordant --- the function is.  A consequence of (generalized) self-concordance is the spectral (multiplicative) stability of Hessian to small perturbations of parameters.

\begin{lemma}[Stability of Hessian{\citep[Theorem~2.1.1]{nesterov1994interior}, \citep[Proposition~1]{bach2010self}}]\label{lem:selfconcordant-hessian}
	Let $H_\theta :=  \nabla^2F_s(\theta)$. If $F_s$ is $R$-self-concordant at $\theta$, then for any $v$ such that $R \|v\|_{H_\theta} < 1$, we have that
	\begin{align*}
	(1-R\|v\|_{H_\theta})^2 \nabla^2 F_s(\theta) 	&\prec	\nabla^2 F_s(\theta+v) \\ &\prec  \frac{1}{(1-R\|v\|_{H_\theta})^2}   \nabla^2 F_s(\theta)  .
	\end{align*}
	If instead we assume $F_s$ is $R$-generalized-self-concordant at $\theta$ with respect to norm $\|\cdot\|$, then
\begin{align*}
	e^{-R\|v\|} \nabla^2 F_s(\theta) \prec  \nabla^2 F_s(\theta+v)  \prec e^{R\|v\|}  \nabla^2 F_s(\theta) 
	\end{align*}
\end{lemma}\label{stability}
The two bounds are almost identical when  $R\|v\|$ and $R\|v\|_{\theta}$ are close to $0$. In particular, for $x\leq 1/2$, we have that $e^{-2x} \leq 1-x \leq e^{-x}$.

In particular, the loss function of binary logistic regression is $1$-generalized self-concordant.
\begin{example}[Binary logistic regression]
	Assume $\|x\|_2\leq 1$ for all $x\in \cX$ and $y\in\{-1,1\}$. Then  binary logistic regression with datasets in $\cX\times \cY$ has a  log-likelihood of 
	$
	F(\theta) = \sum_{i=1}^n \log(1+e^{-y_i x_i^T\theta}).
	$
	The univariate function $l :=  \log(1+\exp(\cdot))$ satisfies 
	$$|l'''|  =  \left|\frac{\exp{(\cdot)}  (1- \exp{(\cdot)})}{(1+\exp{(\cdot)})^3}\right| \leq  \frac{\exp{(\cdot)}}{(1+\exp{(\cdot)})^2} := l''.$$
\end{example}

We next apply the modified output perturbation algorithm to privately release $\theta^*$.
	The algorithm is simply:
	\begin{enumerate}
		\item Solve
			$$
		\theta^* = \argmin_{\theta}  \sum_{i=1}^n l_i(\theta) + r(\theta).
        $$
		\item Release$$
		\hat{\theta} =  \theta^*  +   Z,$$
		where $\gamma>0$ is a tuning parameter and $Z\sim \cN(0, \gamma^{-1} (\sum_{i=1}^n \nabla^2 l_i(\theta)+ \nabla^2 r(\theta))^{-1}).$

	\end{enumerate}

The data-dependent DP of the above procedure is stated as follows.

\begin{theorem}[Data-dependent DP of GLM]\label{thm: glm}
Denote the smooth part of the loss function $F_s=\sum_{i=1}^n l(y_i, <x_i, \cdot>) + r_s(\cdot)$.
 Assume the following:
	\begin{enumerate}
		\item The GLM loss function $l$ is convex, three-times continuously differentiable and $R$-generalized-self-concordant w.r.t. $\|\cdot\|_2$,
		\item $F_s$ is locally $\alpha$-strongly convex w.r.t. $\|\cdot\|_2$,

		\item and in addition, denote $L := \sup_{\theta\in [\theta^*,\tilde{\theta}^*]}|l'(y,x^T\theta)|$, $\beta := \sup_{\theta\in [\theta^*,\tilde{\theta}^*]}|l''(y,x^T\theta)|$. That is, $\ell(\cdot)$ is $L$-Lipschitz and $\beta$-smooth.
	\end{enumerate}
	We then have the data-dependent DP
	$$\epsilon(Z) \leq   \frac{R(L+\beta)}{\alpha} (1+\log(2/\delta))  +  \frac{\gamma L^2}{\alpha}  + \sqrt{\frac{\gamma L^2}{\alpha}\log(2/\delta)}.$$
\end{theorem}

The proof follows by taking an  upper bound of the per-instance DP loss (Theorem~\ref{thm: glm}) $\epsilon(Z,z)$ over $z=(x,y)\in (\cX, \cY)$.

Notice that the Hessians can be arbitrarily singular and $\alpha$ could be $0$, which leads to an infinite privacy loss without additional assumptions. Thus, we will impose an additional regularization of form $\frac{\lambda}{2}||\theta||^2$, which ensures that for any dataset $F_S$ is $\lambda$-strongly convex.

This is not yet DP because it is still about a fixed  dataset. We also need a pre-specified privacy budget $(\epsilon,\delta)$.
We next demonstrate how to apply the generalized PTR to provide a general solution to the above GLM, using logistic regression as an example.

 \begin{remark}[Logistic regression]\label{remark: log}
 	For logistic regression, we know $L\leq 1$, $\beta \leq 1/4$ and if $\|x\|_2 \leq 1$, it is $1$-generalized self-concordant. For any dataset $Z=(X, y)$, the data-dependent DP $\epsilon(X)$ w.r.t. $\delta$ can be simplified to:
 	$$
 	\frac{1.25}{\alpha}(1+\log(2/\delta)) +\frac{\gamma}{\alpha} +   \sqrt{\frac{\gamma}{\alpha}\log(2/\delta)}
 	$$
 \end{remark}
Now, the data-dependent DP is a function of $\alpha$ and $\gamma$, where $\alpha$ denotes the local strong convexity at $\theta_\lambda^*$ and $\gamma$ controls the  noise scale. We next show how to select these two parameters adapted to the dataset. 

\begin{example}
We demonstrate here how we apply  generalized PTR to  output perturbation of the logistic regression problem.  
	\begin{enumerate}
		\item Take an exponential grid of parameters $\{\lambda\}$ and propose each $\lambda$.
		\item Solve for $\theta_{\lambda}^*  = \argmin_{\theta} F(\theta) + \lambda \|\theta\|^2/2$
		\item Calculate the smallest eigenvalue $\lambda_{\min}(\nabla^2F(\theta_\lambda^*))$ (e.g., using power method).
		\item Differentially privately release $\lambda_{\min}$ with $\lambda_{\min}^p:= \max\{\lambda_{\min}+ \frac{\sqrt{\log(4/\delta)}}{\epsilon/2}\cdot \Delta_{GS}\cdot Z -\frac{\sqrt{2\log(4/\delta) \cdot \log(1/\delta)}\Delta_{GS}}{\epsilon/2},0\}$, where $\Delta_{GS}$ denote the global sensitivity of $\lambda_{\min}$
		using Theorem~\ref{thm: gs_lambda}.
		\item Let $\epsilon^p(\cdot)$ be instantiated with $\epsilon(X)$ w.r.t. $\delta$ from Remark~\ref{remark: log}, where $\alpha = \lambda_{\min}^p +\lambda$. 
		Then, conditioned on a high probability event, $\epsilon^p(\cdot)$ (a function of $\gamma$) is a valid DP bound that holds for all datasets and all parameters $\gamma$.
		\item  Calculate the maximum $\gamma$ such that $\epsilon_{\delta/2}^p(\gamma)\leq \epsilon/2$.
		\item Release  $\hat{\theta}  \sim \cN(\theta_\lambda^*, \gamma^{-1}\nabla^2 F_s(\theta_\lambda^*)^{-1})$.
		\item Evaluate the utility on the validation set and return the $(\lambda, \gamma)$ pair that leads to the highest utility.
	\end{enumerate}
\begin{theorem}
For each proposed $\lambda$, the algorithm that releases  $\hat{\theta}  \sim \cN(\theta_\lambda^*, \gamma^{-1}\nabla^2 F_s(\theta_\lambda^*)^{-1})$ is $(\epsilon, 2\delta)$-DP.
\end{theorem}
\begin{proof}
The proof  follows the recipe of generalized PTR with private upper bound (Example~\ref{exp: upperbound}). First, the release of $\lambda_{\min}(\nabla^2 F(\theta_\lambda^*))$ is $(\epsilon/2, \delta/2)$-DP. Then, with probability at least $1-\delta$, $\epsilon_\delta^p(\cdot)> \epsilon_\delta(X)$ holds for all $X$ and $\gamma$.  Finally, $\gamma$ is chosen such that the valid upper bound is $(\epsilon/2, \delta/2)$-DP.
\end{proof}

For the hyper-parameter tuning on $\lambda$ (Steps 1 and 8), we can use Algorithm~\ref{alg: parameter_ptr} to evaluate each $\lambda$.

Unlike Example~\ref{examp:ops}, the $\lambda_\text{min}(\nabla^2 F(\theta_\lambda^*))$ is a complicated data-dependent function of $\lambda$. Thus, we cannot privately release the data-dependent quantity $\lambda_\text{min}(\nabla^2 F(\theta_\lambda^*))$ without an input $\lambda$. The PTR approach allows us to test a number of different $\lambda$ and hence get a more favorable privacy-utility trade-off.
\end{example}
An interesting perspective of this algorithm for logistic regression is that increasing the regularization $\alpha$ is effectively increasing the number of data points within the soft ``margin''\footnote{If we think of logistic regression as a smoothed version of SVM, then increasing $\alpha$ leads to more support vectors. The ``margin'' is ``softer'' in logistic regression, but qualitatively the same.} of separation, hence a larger contribution to the Hessian from the loss function.

\begin{remark}
    The PTR solution for GLMs follows a similar recipe: propose a regularization strength $\lambda$; construct a lower bound of the strong convexity $\alpha$ at the optimal solution $\theta_\lambda^*$; and test the validity of data-dependent DP using Theorem~\ref{thm: glm}.
\end{remark}

Before moving on to other applications of generalized PTR, we will show how to differentially privately release $\lambda_{min}$ according to the requirements of the logistic regression example.

\subsection{Differentially privately release $\lambda_{min}\left(\nabla^2F(\theta)\right)$}

To privately release $\lambda_{min}{\nabla^2F(\theta)}$, we first need to compute its global sensitivity. Once we have that then we can release it differentially privately using either the Laplace mechanism or the Gaussian mechanism.

\begin{theorem}[Global sensitivity of the minimum eigenvalue at the optimal solution]\label{thm: gs_lambda}
	Let $F(\theta) = \sum_{i=1}^n f_i(\theta) +  r(\theta)$ and $\tilde{F}(\theta) = F(\theta) + f(\theta)$ where  $f_1,...,f_n$ are  loss functions corresponding to a particular datapoint $x$.  
	Let $\theta^* = \argmin_\theta F(\theta)$ and $\tilde{\theta}^* = \argmin_\theta \tilde{F}(\theta)$. Assume $f$ is $L$-Lipschitz and $\beta$-smooth, $r(\theta)$ is $\lambda$-strongly convex, and $F$ and $\tilde{F}$ are $R$-self-concordant. If in addition, $\lambda \geq  RL$, then we have
	$$
	\sup_{X,x} (\lambda_{min}(\nabla^2F(\theta_\lambda^*)) -\lambda_{min}(\nabla^2\tilde{F}(\tilde{\theta_\lambda^*})))   \leq   2RL + \beta.
	$$
\end{theorem}
\begin{proof}

\begin{equation}\label{eq:GS_lamb_min_deriv1}
    \begin{split}
 \lambda_{min}&(\nabla^2F(\theta_\lambda^*)) -\lambda_{min}(\nabla^2\tilde{F}(\tilde{\theta_\lambda^*})) \\
 &= (\lambda_{min}(\nabla^2F(\theta_\lambda^*)) -\lambda_{min}(\nabla^2\tilde{F}(\theta_\lambda^*)))  \\&+ (\lambda_{min}(\nabla^2\tilde{F}(\theta_\lambda^*)) -\lambda_{min}(\nabla^2\tilde{F}(\tilde{\theta_\lambda^*}))). 
    \end{split}
\end{equation}

We first bound the part on the left.
By applying Weyl's lemma $\lambda(X+E)-\lambda(X) \leq ||E||_2$,
we have 
\begin{equation}\label{eq:GS_lamb_min_deriv2}
 \sup_{x}||\nabla^2F(\theta_\lambda^*) - \nabla^2\tilde{F(\theta_\lambda^*})||_2 = ||\nabla^2 f(\theta_\lambda^*)||_2 \leq \beta
\end{equation}
In order to bound the part on the right, we apply the semidefinite ordering  using self-concordance, which gives 
$$
e^{-R\|\tilde{\theta_\lambda^*} - \theta_\lambda^*\|} \nabla^2\tilde{F}(\tilde{\theta_\lambda^*})\prec \nabla^2\tilde{F}(\theta_\lambda^*) \prec e^{R\|\tilde{\theta_\lambda^*} - \theta_\lambda^*\|} \nabla^2\tilde{F}(\tilde{\theta_\lambda^*}).
$$
By the Courant-Fischer Theorem and the monotonicity theorem, we also have that for the smallest eigenvalue
\begin{equation}
\label{eq:GS_lamb_min_deriv3}
\begin{split}
e^{-R\|\tilde{\theta_\lambda^*} - \theta_\lambda^*\|} &  \lambda_{\min}\left(\nabla^2\tilde{F}(\tilde{\theta_\lambda^*})\right)  \leq    \lambda_{\min}\left(\nabla^2\tilde{F}(\theta_\lambda^*) \right) \\&\leq e^{R\|\tilde{\theta_\lambda^*} - \theta_\lambda^*\|}   \lambda_{\min}\left(\nabla^2\tilde{F}(\tilde{\theta_\lambda^*})\right).
\end{split}
\end{equation}

Moreover by Proposition~\ref{prop:generalnorm}, we have that 
$$
\|\tilde{\theta_\lambda^*} - \theta_\lambda^*\|_2 \leq  \frac{\|\nabla f(\tilde{\theta^*}_\lambda)\| }{\lambda_{\min}\left(\nabla^2\tilde{F}(\tilde{\theta_\lambda^*})\right)}  \leq \frac{L}{\lambda_{\min}\left(\nabla^2\tilde{F}(\tilde{\theta_\lambda^*})\right)}.
$$

If $ \lambda_{\min}\left(\nabla^2\tilde{F}(\tilde{\theta_\lambda^*})\right) \geq RL$,  then use that $e^x-1 \leq 2x$ for $x\leq 1$. Substituting the above bound to \eqref{eq:GS_lamb_min_deriv3} then to \eqref{eq:GS_lamb_min_deriv1} together with \eqref{eq:GS_lamb_min_deriv2}, we get a data-independent global sensitivity bound of
$$\lambda_{min}(\nabla^2F(\theta_\lambda^*)) -\lambda_{min}(\nabla^2\tilde{F}(\tilde{\theta_\lambda^*}))  \leq 2RL + \beta$$
as stated.
\end{proof}

\begin{proposition}\label{prop:generalnorm}
	Let $\|\cdot\|$ be a norm and $\|\cdot\|_*$ be its dual norm. Let $F(\theta)$, $f(\theta)$ and $\tilde{F}(\theta) = F(\theta) + f(\theta)$ be proper convex functions and $\theta^*$ and $\tilde{theta}^*$ be their minimizers, i.e., $0\in \partial F(\theta^*)$ and $0\in \partial \tilde{F}(\tilde{theta}^*)$.  If in addition, $F,\tilde{F}$ is $\alpha,\tilde{\alpha}$-strongly convex with respect to $\|\cdot\|$ within the restricted domain 
	$\theta \in \{  t\theta^* + (1-t)\tilde{\theta}^*  \;|\;  t\in[0,1]  \}$. 	Then there exists $g \in \partial f(\theta^*)$ and $\tilde{g}\in \partial f(\tilde{\theta}^*)$ such that
	$$
	\|\theta^*-\tilde{\theta}^*\| \leq\min\left\{\frac{1}{\alpha}\| \tilde{g}\|_*,  \frac{1}{\tilde{\alpha}}\| g\|_*\right\}.
	$$
\end{proposition}
\begin{proof}
	Apply the first order condition to $F$ restricted to the line segment between $\tilde{\theta}^*$ and $\theta^*$, we get
	\begin{align}
	F(\tilde{\theta}^*) \geq F(\theta^*)  +  \langle \partial F(\theta^*),  \tilde{\theta}^*-\theta^* \rangle  + \frac{\alpha}{2}\|\tilde{\theta}^*-\theta^*\|^2\label{eq:strongcvx1} \\
	F(\theta^*) \geq F(\tilde{\theta}^*)  +  \langle \partial F(\tilde{\theta}^*),  \theta^*-\tilde{\theta}^* \rangle  + \frac{\alpha}{2}\|\tilde{\theta}^*-\theta^*\|^2 \label{eq:strongcvx2}
	\end{align}
	Note by the convexity of $F$ and $f$, $\partial\tilde{F}=  \partial F + \partial f$, where $+$ is the Minkowski Sum. Therefore, $0\in \partial\tilde{F}(\tilde{\theta}^*)$ implies that there exists $\tilde{g}$ such that $\tilde{g}\in \partial f(\tilde{\theta}^*)$ and $-\tilde{g}\in\partial F(\tilde{\theta}^*)$.
	Take $-\tilde{g}\in\partial F(\tilde{\theta}^*)$ in Equation~\ref{eq:strongcvx2} and $0 \in \partial F(\theta^*)$ in Equation~\ref{eq:strongcvx1}  and add the two inequalities, we obtain
	\begin{align*}
		0&\geq \langle -\tilde{g},  \theta^*-\tilde{\theta}^* \rangle  + \alpha \|\tilde{\theta}^* - \theta^*\|^2 \\ &\geq - \|\tilde{g}\|_* \|\theta^*-\tilde{\theta}^*\|  +  \alpha\|\tilde{\theta}^* - \theta^*\|^2. 
	\end{align*}
	For $\|\tilde{\theta}^* - \theta^*\|=0$ the claim is trivially true; otherwise, we can divide  both sides of the above inequality by $\|\tilde{\theta}^* - \theta^*\|$ and get
	$	\|\theta^*-\tilde{\theta}^*\| \leq \frac{1}{\alpha}\| \tilde{g}\|_*$. 
	
	It remains to show that $\|\theta^*-\tilde{\theta}^*\| \leq \frac{1}{\tilde{\alpha}}\|g\|_*$. This can be obtained by exactly the same arguments above but applying strong convexity to $\tilde{F}$ instead. Note that we can actually get something slightly stronger than the statement because the inequality holds for all $g\in \partial f(\theta^*)$.
\end{proof}


\subsection{Other applications of generalized PTR}

Besides one-posterior sampling for GLMs, there are plenty of examples that our generalized-PTR could be applied, e.g., DP-PCA~\citep{dwork2014analyze} and Sparse-DP-ERM~\citep{kifer2012private} (when the designed matrix is well-behaved).

\citep{dwork2014analyze} provides a PTR style privacy-preserving principle component analysis (PCA).
The key observation of \citep{dwork2014analyze} is that the local sensitivity is quite ``small'' if there is a large eigengap between the $k$-th and the $k+1$-th eigenvalues. Therefore, their approach (Algorithm 2) chooses to privately release a lower bound of the k-th eigengap ($k$ is fixed as an input) and use that to construct a high-confidence upper bound of the local sensitivity.

For noise-adding mechanisms, the local sensitivity is proportional to the data-dependent loss and generalized PTR is applicable. We can formulate the data-dependent DP of DP-PCA as follows:

\begin{theorem}~\label{thm: pca}
For a given matrix $A\in \cR^{m \times n}$, assume each row of $A$ has a bounded $\ell_2$ norm being $1$.  Let $V_k$ denotes the top $k$ eigenvectors of $A^TA$ and $d_k$ denotes the gap between the $k$-th and the $k+1$-th eigenvalue. Then releasing $V_k V_k^T + E$, where $E \in \cR^{n\times n}$ is a symmetric matrix with the upper triangle is i.i.d samples from $\cN(0, \sigma^2)$ satisfies $(\epsilon(A), \delta)$ data-dependent DP and $\epsilon(A) = \frac{2\sqrt{\log(1.25/\delta)}}{\sigma(d_k -2)} $. 
\end{theorem}
The proof is based on the local sensitivity result from \citep{dwork2014analyze} and the noise calibration of Gaussian mechanism.

We can combine Theorem~\ref{thm: pca} with our Algorithm~\ref{alg: parameter_ptr} to instantiate the generalized PTR framework. The improvement over \citet{dwork2014analyze} will be to allow joint tuning of the parameter $k$ and the noise variance (added to the spectral gap $d_k$).

\section{Omitted proofs in Section~\ref{sec:gen_ptr}}\label{sec: omit_gen_ptr}

The utility of Algorithm~\ref{alg: parameter_ptr}
depends on how many rounds that Algorithm~\ref{alg:gen_ptr} is invoked. We next provide the utility guarantee of Algorithm~\ref{alg: parameter_ptr}, which follows a simplification of the result in the Section A.2 of~\citet{papernot2021hyperparameter}.
\begin{theorem}
 Suppose applying Algorithm~\ref{alg:gen_ptr} with each $\phi_i$ has an equal probability to achieve the highest validation score. Let $\hat{T}$ denotes the number of invocation of Algorithm~\ref{alg:gen_ptr}, where $\hat{T}$ follows a truncated geometric distribution. Then the expected quantile of the highest score candidate is given by $\mathbb{E}_{\hat{T}}\bigg[1 -\frac{1}{\hat{T}+1}\bigg]$.
\end{theorem}

In practice, we can roughly set $\tau = \frac{1}{10k}$ so that the algorithm is likely to test all $k$ parameters. 


\begin{proof}
 Suppose each oracle access to $Q(X)$ has a probability $1/k$ of achiving the best validation accuracy. Let $\beta$ denote the probability that $\cA$ (shorthand for Algorithm~\ref{alg: parameter_ptr}) outputs the best choice of $\phi_i$. 
 \begin{align*}
 \beta &= 1-\pr[\cA(X) \text{is not best}] \\ &=1- \mathbb{E}_{\hat{T}}\bigg[\pr[Q(X) \text{is not best}]^{\hat{T}}\bigg] \\ &= 1 -\mathbb{E}_{\hat{T}}\bigg[ (1-\frac{1}{k})^{\hat{T}}\bigg].
 \end{align*}
 Let $f(x)=\mathbb{E}[x^{\hat{T}}]$. Applying a first-order approximation on $f(1-\frac{1}{k})$, we have $f(1-\frac{1}{k})\approx f(1) -f'(1) \cdot \frac{1}{k}=1-\mathbb{E}[\hat{T}]/k$. Then, if $k$ is large and we choose $\tau=0.1/k$, $\cA$ can roughly return the best $\phi_i$.
\end{proof}

\section{Experimental details}

\subsection{Experimental details in private linear regression}

We start with the privacy calibration of the OPS-PTR algorithm.

\begin{algorithm}[t]
	\caption{OPS-PTR: One-Posterior Sample with propose-test-release (no-``perp'' version)}
	\label{alg: ops_ptr}
	\begin{algorithmic}[1]
		\STATE {\textbf{Input}: Data $X, \bf{y}$. Private budget : $\epsilon, \delta$, proposed regularizer $\lambda$. }
		\STATE{Calculate the minimum eigenvalue $\lambda_{\text{min}}(X^TX)$.}
		\STATE{Sample $Z \sim \cN(0, 1)$ and privately release $\tilde{\lambda}_{\text{min}}=\text{max}\bigg\{
		\lambda_{\text{min}}+ \frac{\sqrt{\log(6/\delta)}}{\epsilon/4}Z -\frac{\sqrt{2\log(6/\delta)\cdot \log(2/\delta)}}{\epsilon/4},0\bigg\}$}
		\STATE{Calculate $\hat{\theta}=(X^TX+\lambda I)^{-1}X^Ty$.}
		\STATE{Sample $Z\sim \cN(0, 1)$ and privately release $\Delta = \log(||\cY||+||\cX||||\hat{\theta}||)+ \frac{\log(1+||\cX||^2/(\lambda +\tilde{\lambda}_{\text{min}}))}{\epsilon/(4\sqrt{6/\delta})}Z + \frac{\log(1+||\cX||^2/(\lambda +\tilde{\lambda}_{\text{min}}))}{\epsilon/(4\sqrt{2\log(6/\delta)\log(2/\delta)})}$.}
		\STATE{Set the local Lipschitz $\tilde{L}:=||X||e^\Delta.$}
		\STATE{ Calibrate $\gamma$ with Theorem~\ref{thm: per}($\delta/3, \epsilon/2$.) }
		\STATE{Output $\tilde{\theta}\sim p(\theta|X,\bf{y}) \propto e^{-\frac{\gamma}{2}||\bf{y}-X\theta||^2 + \lambda||\theta||^2}$}
	\end{algorithmic}
\end{algorithm}

Algorithm~\ref{alg: ops_ptr} provides the detailed privacy calibration of the private linear regression problem.

\begin{theorem}
Algorithm~\ref{alg: ops_ptr} is $(\epsilon, 2\delta)$-DP.
\end{theorem}
\begin{proof}
There are three data-dependent quantities in Theorem~\ref{thm: per}: $\lambda_{
\text{min}}, ||\theta_\lambda^*||$ and $L$.
First, notice  that $\lambda_{\text{min}}$ has a global sensitivity of $||\cX||^2$ by Weyl's lemma. Under the assumption $||\cX||^2\leq 1$, we privately release $\lambda_{\text{min}}$ using $(\epsilon/4,\delta/3)$ in Step~3.
Notice that with probability at least $1-\delta/2$, $\tilde{\lambda}_{\min}$ is a lower bound of $\lambda_{\min}$. 

Then, we apply Lemma~\ref{lem: adaops_ls} from ~\citet{wang2018revisiting} to privately release $\log(||\cY||+||\cX||||\hat{\theta}||)$ using $(\epsilon/4, \delta/3)$. Note that both the local Lipschitz constant $L$ and the norm  $||\theta_\lambda^*||$ are functions of $\log(||\cY||+||\cX||||\hat{\theta}||)$. Thus, we can construct a private upper bound of these by post-processing of $\Delta$.

Then, with probability at least $1-\delta$ (by a union bound over $\tilde{\lambda}_{\min}$ and $\Delta$), instantiating Theorem~\ref{thm: per} with $\tilde{\lambda}_{\min}$ and $\tilde{L}$ provides a valid upper bound of the data-dependent DP.
We then tune the parameter $\gamma$ using the remaining privacy budget $(\epsilon/2, \delta/3)$.
\end{proof}

\begin{lemma}[Lemma 12~\citep{wang2018revisiting}]\label{lem: adaops_ls}
Let $\theta_\lambda^*$ be the ridge regression estimate with parameter $\lambda$ and the smallest eigenvalue of $X^TX$ be $\lambda_{\text{min}}$, then the function $\log(||\cY +||\cX||||\theta_\lambda^*||)$ has a local sensitivity of $\log(1+\frac{||\cX||^2}{\lambda_{\text{min}+\lambda}})$.
\end{lemma}

%% file: pate_appendix.tex
\subsection{Details of PATE case study}

\begin{definition}[Renyi DP \citep{mironov2017renyi}]
    We say a randomized algorithm $\cM$ is $(\alpha, \epsilon_\cM(\alpha))$-RDP with order $\alpha \geq 1$ if for neighboring datasets $X, X'$
    \begin{align*}
    &\mathbb{D}_{\alpha}(\cM(X)||   \cM(X')):=\\
    & \frac{1}{\alpha-1}\log \mathbb{E}_{o \sim \cM(X')}\bigg[ \bigg( \frac{\pr[\cM(X)=o]}{\pr[\cM(X')=o]}\bigg)^\alpha \bigg]\leq \epsilon_\cM(\alpha).
    \end{align*}
\end{definition}
At the limit of $\alpha \to \infty$, RDP reduces to $(\epsilon, 0)$-DP. 
We now define the  data-dependent Renyi DP that conditioned on an input dataset $X$.
\begin{definition}[Data-dependent Renyi DP \citep{papernot2018scalable}]
    We say a randomized algorithm $\cM$ is $(\alpha, \epsilon_\cM(\alpha, X))$-RDP with order $\alpha \geq 1$ for dataset $X$ if for neighboring datasets $X'$
    \begin{align*}
    &\mathbb{D}_{\alpha}(\cM(X)||   \cM(X')):=\\
    & \frac{1}{\alpha-1}\log \mathbb{E}_{o \sim \cM(X')}\bigg[ \bigg( \frac{\pr[\cM(X)=o]}{\pr[\cM(X')=o]}\bigg)^\alpha \bigg]\leq \epsilon_\cM(\alpha, X).
    \end{align*}
\end{definition}


RDP features two useful properties.
\begin{lemma}[Adaptive composition]
    $\epsilon_{(\cM_1, \cM_2)} = \epsilon_{\cM_1}(\cdot) + \epsilon_{\cM_2}(\cdot)$.
\end{lemma}
\begin{lemma}[From RDP to DP] If a randomized algorithm $\cM$ satisfies $(\alpha,\epsilon(\alpha))$-RDP, then $\cM$ also satisfies $(\epsilon(\alpha)+\frac{\log(1/\delta)}{\alpha-1},\delta)$-DP for any $\delta \in (0,1)$. \label{lem: rdp2dp}
\end{lemma}

\begin{definition}[Smooth Sensitivity]\label{def: smooth}
	Given the smoothness parameter $\beta$, a $\beta$-smooth sensitivity of $f(X)$ is defined as 
	\[SS_\beta(X):= \max_{d\geq 0} e^{-\beta d} \cdot \max_{\tilde{X'}: dist(X, \tilde{X'})\leq d} \Delta_{LS}(\tilde{X}')\]
\end{definition}

\begin{lemma}[Private upper bound of data-dependent  RDP, Restatement of Theorem~\ref{lem: upperbound}]]
Given a RDP function $\rdp(\alpha, X)$ and a $\beta$-smooth sensitivity bound $SS(\cdot)$ of $\rdp(\alpha, X)$. Let $\mu$ (defined in Algorithm~\ref{alg: pate_ptr}) denote the private release of $\log(SS_\beta(X))$. Let $(\beta, \sigma_s, \sigma_2)$-GNSS mechanism be 
\[\scriptstyle
\rdp^{\text{upper}}(\alpha):=\rdp(\alpha, X) + SS_\beta(X) \cdot \cN(0, \sigma_s^2) + \sigma_s \sqrt{2\log(\frac{2}{\delta_2}) } e^{\mu} \]
	 Then, the release of $\rdp^{\text{upper}}(X)$ satisfies $(\alpha, \frac{3\alpha +2}{2\sigma_s^2})$-RDP for all $1<\alpha < \frac{1}{2\beta}$; w.p. at least $1-\delta_2$, $\rdp^{\text{upper}}(\alpha)$ is an upper bound of $\rdp(\alpha, X)$.
\vspace{-2mm}
\end{lemma}

\begin{proof}[Proof sketch]
  
We first show that releasing the smooth sensitivity $SS_\beta$ with $e^\mu$ satisfies $(\alpha, \frac{\alpha}{2\sigma_2^2})$-RDP. Notice that the log of $SS_\beta(X)$ has a bounded global sensitivity $\beta$ (Definition~\ref{def: smooth} implies that $|\log SS_\beta(X)-\log SS_\beta(X')|\leq \beta $ for any neighboring dataset $X, X'$). By Gaussian mechanism, scaling noise with $\beta \sigma_2$ to $\log SS_\beta(X)$ is $(\alpha, \frac{\alpha}{2\sigma_2^2})$-RDP.
Therefore, the release of $\rdp(\alpha, X)$ is $(\alpha, \epsilon_s(\alpha)+\frac{\alpha}{2\sigma_2^2})$-RDP. Since the release  of $ f(X) + SS_\beta(X)\cdot \cN(0, \sigma_s^2)$ is $(\alpha, \frac{\alpha+1}{\sigma_s^2})$-RDP (Theorem 23 from \citet{papernot2018scalable}) for $\alpha<\frac{1}{2\beta}$, we have
$\epsilon_s(\alpha)+\frac{\alpha}{2\sigma_2^2}=\frac{3\alpha+2}{2\sigma_s^2}$.

We next prove the second statement. First, notice that with probability at least $1-\delta_2/2$, $e^\mu \geq SS_\beta(X)$ using the standard Gaussian tail bound.  Let $E$ denote the event that $e^{\mu}\geq SS_\beta(X)$.


\begin{align*}
   & \pr\bigg[\rdp^{\text{upper}}(\alpha)\leq \rdp(\alpha, X)\bigg] \\
   &=  \pr\bigg[\rdp^{\text{upper}}(\alpha) \leq \rdp(\alpha,X)|E\bigg] + \pr\bigg[\rdp^{\text{upper}}(\alpha)\leq \rdp(\alpha, X)|E^c\bigg]\\
   &\leq \pr\bigg[\rdp^{\text{upper}}(\alpha) \leq\rdp(\alpha, X)|E\bigg] + \delta_2/2\\
   &= \underbrace{\pr\bigg[\cN(0, \sigma_s^2)\cdot SS_{\beta(X)}\geq \sigma_s \cdot \sqrt{2\log(2/\delta_2)}e^{\mu} |E\bigg]}_{\text{denoted by} (*)} + \delta_2/2\\
\end{align*}

Condition on the event $E$, $e^{\mu}$ is a valid upper bound of $SS_\beta(X)$, which implies  \[ (*) \leq  \pr[\cN(0, \sigma_s^2)\cdot SS_\beta(X) \geq \sigma_s \cdot \sqrt{2\log(2/\delta_2)} SS_\beta(X) |E] \leq \delta_2/2\]
Therefore, with probability at least $1- \delta_2$, $\rdp^{\text{upper}}(\alpha) \geq \rdp(\alpha, X)$.
\end{proof}

\begin{theorem}[Restatement of Theorem~\ref{thm: pate_ptr}]
Algorithm~\ref{alg: pate_ptr} satisfies $(\epsilon'+\hat{\epsilon}, \delta)$-DP.
\end{theorem}

\begin{proof}
The privacy analysis consists of two components --- the privacy cost of releasing an upper bound of data-dependent RDP ($\epsilon_{\text{upper}}(\alpha):=\epsilon_s(\alpha)+\frac{\alpha}{2\sigma_2^2}$ and the valid upper bound $\epsilon_{\sigma_1}^p(\alpha)$.
First, set $\alpha =\frac{2\log(2/\delta)}{\epsilon}+1$ and use RDP to DP conversion with $\delta/2$ ensures that the cost of $\delta/2$ contribution to be roughly $\epsilon/2$ (i.e., $\frac{\log(2/\delta)}{\alpha-1} = \epsilon/2$). Second,  choosing $\sigma_s = \sqrt{\frac{2+3\alpha}{\epsilon}}$ gives us another $\epsilon/2$. 
\end{proof}


\textbf{Experimental details}
 $K=400$ teacher models are trained individually on the disjoint set using AlexNet model. We set $\sigma_2 = \sigma_s = 15.0$.   Our data-dependent RDP calculation and the smooth-sensitivity calculation follow \citet{papernot2018scalable}. Specifically, we use the following theorem (Theorem~6 from~\citet{papernot2018scalable}) to compute the data-dependent RDP of each unlabeled data $x$ from the public domain.

\begin{theorem}[data-dependent RDP ~\citet{papernot2018scalable}]
 Let $\tilde{q}\geq \pr[\cM(X)\neq \argmax_{j\in [C]} n_j(x)]$, i.e., an upper bound of the probability that the noisy label does not match the majority label. Assume $\alpha\leq \mu_1$ and $\tilde{q}\leq e^{(\mu_2 -1)\epsilon_2}/\bigg(\frac{\mu_1}{\mu_1 -1} \cdot \frac{\mu_2}{\mu_2 -1}\bigg)^{\mu_2}$, then we have:
 \[\epsilon_{\cM}(\alpha, X) \leq \frac{1}{\alpha-1}\log \bigg( (1-\tilde{q})\cdot A(\tilde{q}, \mu_2, \epsilon_2)^{\alpha-1} +\tilde{q}\cdot B(\tilde{q}, \mu_1, \epsilon_1)^{\alpha-1}\bigg)  \]
 where $A(\tilde{q}, \mu_2, \epsilon_2):= (1-\tilde{q})/\bigg(1-(\tilde{q}e^{\epsilon_2})^{\frac{\mu_2-1}{\mu_2}}\bigg)$, $B(\tilde{q},\mu_1, \epsilon_1)=e^{\epsilon_1}/\tilde{q}^{\frac{1}{\mu_1 -1}}, \mu_2=\sigma_1 \cdot \sqrt{\log(1/\tilde{q})}, \mu_1 = \mu_2 +1, \epsilon_1 = \mu_1/\sigma_1^2 $ and $\epsilon_2 = \mu_2/\sigma_2^2$.
    
\end{theorem}
 
In the experiments, the non-private data-dependent DP baseline is also based on the above theorem.  Notice that the data-dependent RDP of each query is a function of $\tilde{q}$, where $\tilde{q}$ denotes an upper bound of the probability where the plurality output does not match the noisy output. $\tilde{q}$ is a complex function of both the noisy scale and data and is not monotonically decreasing when $\sigma_1$ is increasing.

\textbf{Simulation of two distributions.}
The motivation of the experimental design is to compare three approaches under different data distributions. 
Notice that there are $K=400$ teachers, which implies the number of the vote count for each class will be bounded by $400$. In the simulation of high-consensus distribution, we choose $T=200$ unlabeled public data such that the majority vote count will be larger than $150$ (i.e., $\max_{j\in[C]} n_j(x)>150$). For the low-consensus distribution, we choose to select $T$ unlabeled data such that the majority vote count will be smaller than $150$.

\section{Omitted proofs in private GLM}
\subsection{Per-instance DP of GLM}
\begin{theorem}[Per-instance differential privacy guarantee\label{thm: glm}]
	Consider two adjacent data sets $Z$ and $Z' =[Z, (x,y)]$, and denote the smooth part of the loss function $F_s =   \sum_{i=1}^n l(y_i,\langle x_i, \cdot\rangle) + r_s(\cdot)$ (thus $\tilde{F}_s = F_s  +  l(y,\langle x, \cdot \rangle)$.
	Let the local neighborhood be the line segment between $\theta^*$ and $\tilde{\theta}^*$. Assume 
	\begin{enumerate}
		\item the GLM loss function $l$ be convex, three-time continuous differentiable and $R$-generalized-self-concordant w.r.t. $\|\cdot\|_2$,
		\item $F_s$ is locally $\alpha$-strongly convex w.r.t. $\|\cdot\|_2$,
		\item and in addition, denote $L := \sup_{\theta\in [\theta^*,\tilde{\theta}^*]}|l'(y,x^T\theta)|$, $\beta := \sup_{\theta\in [\theta^*,\tilde{\theta}^*]}|l''(y,x^T\theta)|$.	
	\end{enumerate}
	
	Then the algorithm obeys $(\epsilon,\delta)$-pDP for $Z$ and $z=(x,y)$ with any $0<\delta < 2/e$ and
$$
\epsilon \leq \epsilon_0(1+\log(2/\delta))  +  e^{\frac{RL\|x\|_2}{\alpha}} \left[\frac{\gamma L^2\|x\|_{H^{-1}}^2}{2} +  \sqrt{ \gamma L^2\|x\|_{H^{-1}}^2\log(2/\delta) }\right]
$$
where 
$\epsilon_0 \leq e^{\frac{RL\|x\|_2}{\alpha}} -1  + 2\beta \|x\|_{H_1^{-1}}^2 +  2\beta\|x\|_{\tilde{H}_2^{-1}}^2.$
If we instead assume that $l$ is $R$-self concordant. Then the same results hold, but with all $e^{\frac{RL\|x\|_2}{\alpha}}$ replaced with $(1-RL\|x\|_{H^{-1}})^2$.

\end{theorem}
	
	Under the stronger three-times continuous differentiable assumption, by mean value theorem, there exists $\xi$ on the line-segment between $\theta^*$ and $\ttheta^*$ such that 
	$$
	H = \left[\int_{t=0}^{1}\nabla^2 F_s((1-t)\theta^* + t\ttheta^*)  dt \right]  =  \nabla^2 F_s(\xi).
	$$
	
	The two distributions of interests are $\cN(\theta^*,  [\gamma \nabla^2 F_s(\theta^*)]^{-1})$ and $\cN(\ttheta^*, [\gamma \nabla^2 F_s(\ttheta^*) + \nabla^2l(y,x^T\ttheta^*)]^{-1}).$
	Denote $[\nabla^2 F_s(\theta^*)]^{-1} =: \Sigma$ and $[\nabla^2 F_s(\ttheta^*) + \nabla^2l(y,x^T\ttheta^*)]^{-1} =: \tilde{\Sigma}$.
	Both the means and the covariance matrices are different, so we cannot use multivariate Gaussian mechanism naively. Instead we will take the tail bound interpretation of $(\epsilon,\delta)$-DP and make use of the per-instance DP framework as internal steps of the proof. 
	
	First, we can write down the privacy loss random variable in analytic form
	\begin{align*}
	\log\frac{|\Sigma|^{-1/2}e^{- \frac{\gamma}{2}\|\theta -\theta^*\|_{\Sigma^{-1}}^2}}{|\tilde{\Sigma}|^{-1/2}e^{- \frac{\gamma}{2}\|\theta -\ttheta^*\|_{\tilde{\Sigma}^{-1}}^2}}
	=\underbrace{\frac{1}{2}\log \left(\frac{|\Sigma^{-1}|}{|\tilde{\Sigma}^{-1}|}\right)}_{(*)} +  \underbrace{\frac{\gamma}{2}\left[\|\theta -\theta^*\|_{\Sigma^{-1}}^2 - \|\theta -\ttheta^*\|_{\tilde{\Sigma}^{-1}}^2\right]}_{(**)}
	\end{align*}
	The general idea of the proof is to simplify the expression above and  upper bounding the two terms separately using self-concordance and matrix inversion lemma, and ultimately show that the privacy loss random variable is dominated by another random variable having an appropriately scaled shifted $\chi$-distribution, therefore admits a Gaussian-like tail bound.

	To ensure the presentation is readable, we define a few short hands. We will use $H$ and $\tilde{H}$ to denote the Hessian of $F_s$ and $F_s +  f$ respectively and subscript $1$ $2$ indicates whether the Hessian evaluated at at $\theta^*$ or $\ttheta^*$. $H$ without any subscript or superscript represents the Hessian of $F_s$ evaluated at $\xi$ as previously used.
	\begin{align*}
	(*)  = \frac{1}{2} \log  \frac{|H_1|}{ |H| }\frac{|H|}{|H_2|}\frac{|H_2|}{|\tilde{H}_2|}  \leq \frac{1}{2}\left[  	\log\frac{|H_1|}{ |H| }  + \log \frac{|H|}{|H_2|} + \log\frac{|H_2|}{|\tilde{H}_2|}\right]
	\end{align*}
	By the $R$-generalized self-concordance of $F_s$, we can apply Lemma~\ref{lem:selfconcordant-hessian}, 
	$$
-\|\theta^*-\xi\|_2R\leq \log\frac{|H_1|}{ |H| } \leq R\|\theta^*-\xi\|_2, \quad   -R\|\xi - \ttheta^*\|_2\leq \log\frac{|H|}{ |H_2| } \leq R\|\xi - \ttheta^*\|_2.
	$$
	The generalized linear model ensures that the Hessian of $f$ is rank-$1$:
	$$\nabla^2 f(\ttheta^*) =  l''(y,x^T\ttheta^*)  xx^T$$
	and we can apply Lemma~\ref{lem:determinant} in both ways (taking $A=H_2$ and $A=\tilde{H}_2$) and obtain
	$$
	\frac{|H_2|}{|\tilde{H}_2|}   =  \frac{1}{1 + l''(y,x^T\ttheta^*)x^T H_2^{-1}  x}  =  1- l''(y,x^T\ttheta^*)x^T\tilde{H}_2 x
	$$
	Note that $ l''(y,x^T\ttheta^*)x^T\tilde{H}_2^{-1} x$ is the in-sample leverage-score and $ l''(y,x^T\ttheta^*)x^T H_2^{-1}  x$ is the out-of-sample leverage-score of the locally linearized problem at $\ttheta^*$. We denote them by $\mu_2$ and $\mu'_2$ respectively (similarly, for the consistency of notations, we denote the in-sample and out of sample leverage score at $\theta^*$ by $\mu_1$ and $\mu'_1$ ). 
	
Combine the above arguments we get
	\begin{align}\label{eq:der_part1}
	   (*)\leq&  R\|\theta^*-\xi\|_2 + R\|\xi - \ttheta^*\|_2  + \log (1 - \mu_2) \leq R\|\theta^*-\ttheta^*\|_2 + \log(1-\mu_2)\\
	   (*) \geq& -R\|\theta^*-\ttheta^*\|_2  - \log(1-\mu_2).
	\end{align}
	
We now move on to deal with the second part, where we would like to express everything in terms of $\|\theta-\theta^*\|_{H_1}$, which we know from the algorithm is $\chi$-distributed.
\begin{align*}
(**)  = \frac{\gamma}{2}\left[ \|\theta -\theta^*\|_{H_1}^2 - \|\theta -\theta^*\|_{H_2}^2  + \|\theta -\theta^*\|_{H_2}^2 - \|\theta -\ttheta^*\|_{H_2}^2+ \|\theta -\ttheta^*\|_{H_2}^2- \|\theta -\ttheta^*\|_{\tilde{H}_2}^2  \right]
\end{align*}
By the generalized self-concordance at $\theta^*$ 
\begin{align*}
e^{-R\|\theta^*-\ttheta^*\|_2}\|\cdot\|_{H_1}^2 \leq \|\cdot\|_{H_2}^2 \leq   e^{R\|\theta^*-\ttheta^*\|_2}\|\cdot\|_{H_1}^2
\end{align*}
This allows us to convert from $\|\cdot\|_{H_2}$ to $\|\cdot\|_{H_1}$, and as a consequence:
$$
\left|\|\theta -\theta^*\|_{H_1}^2 - \|\theta -\theta^*\|_{H_2}^2 \right|  \leq   [e^{R\|\theta^*-\ttheta^*\|_2} - 1]\|\theta -\theta^*\|_{H_1}^2.
$$
Also, 
\begin{align*}
 \|\theta -\theta^*\|_{H_2}^2 - \|\theta -\ttheta^*\|_{H_2}^2  &=  \left\langle \ttheta^* -\theta^* ,  2\theta - 2\theta^* + \theta^*-\ttheta^*  \right\rangle_{H_2}  =  2 \langle  \theta-\theta^*, \ttheta^* -\theta^* \rangle_{H_2} -  \|\theta^*-\ttheta^*\|_{H_2}^2
 \end{align*}
 Therefore
 \begin{align*}
 \left|  \|\theta -\theta^*\|_{H_2}^2 - \|\theta -\ttheta^*\|_{H_2}^2\right|  &\leq 2\|\theta - \theta^*\|_{H_2} \|\theta^*-\ttheta^*\|_{H_2}  +  \|\theta^*-\ttheta^*\|_{H_2}^2  \\
 &\leq 2e^{R\|\ttheta^* - \theta^*\|_2}\|\theta - \theta^*\|_{H_1} \|\theta^*-\ttheta^*\|_{H}  + e^{R\|\ttheta^* - \theta^*\|_2}\|\theta^*-\ttheta^*\|_{H}^2.
\end{align*}
Then lastly  we have
\begin{align*}
0\geq \|\theta -\ttheta^*\|_{H_2}^2- \|\theta -\ttheta^*\|_{\tilde{H}_2}^2 &=  -l''(y,x^T\ttheta^*)\left[ \langle x, \theta-\theta^* \rangle + \langle x,\theta^*-\ttheta^*\rangle\right]^2   \\
&\geq -2\beta \|x\|_{H_1^{-1}}^2\|\theta-\theta^*\|_{H_1}^2   -  2\beta \|x\|_{H^{-1}}^2\|\theta^*-\ttheta^*\|_{H}^2
\end{align*}
$$
\left|  \|\theta -\ttheta^*\|_{H_2}^2- \|\theta -\ttheta^*\|_{\tilde{H}_2}^2\right|  \leq 2\beta \|x\|_{H_1^{-1}}^2\|\theta-\theta^*\|_{H_1}^2   +  2\beta \|x\|_{H^{-1}}^2\|\theta^*-\ttheta^*\|_{H}^2
$$

Combine the above derivations, we get 
\begin{align}
\left|(**)\right|  \leq \frac{\gamma}{2}\left[  a \|\theta-\theta^*\|_{H_1}^2 + b \|\theta-\theta^*\|_{H_1}  +c\right] \label{eq:der_part2}
\end{align}
where 
\begin{align*}
a :=& \left[ e^{R\|\theta^*-\ttheta^*\|_2} -1  + 2\beta \|x\|_{H_1^{-1}}^2\right] \\
b:=& 2 e^{R\|\theta^*-\ttheta^*\|_2}   \|\theta^*-\ttheta^*\|_H \\
c:=& (e^{R\|\theta^*-\ttheta^*\|_2} + 2\beta \|x\|_{H^{-1}}^2)\|\theta^*-\ttheta^*\|_H^2
\end{align*}

Lastly, by \eqref{eq:der_part1} and $\eqref{eq:der_part2}$, 
$$
\left|  \log\frac{p(\theta|Z)}{p(\theta|Z')}  \right|  \leq R\|\theta^*-\ttheta^*\|_2  + \log(1-\mu_2)  +  \frac{\gamma}{2} [ a W^2 + bW + c].
$$
where according to the algorithm $W:= \|\theta-\theta^*\|_{H_1}$ follows a half-normal distribution with $\sigma=\gamma^{-1/2}$.

By standard Gaussian tail bound, we have for all $\delta<2/e$.
$$
\P(|W|\leq \gamma^{-1/2} \sqrt{\log(2/\delta)} )  \leq \delta.
$$
This implies that a high probability upper bound of the absolute value of the privacy loss random variable $\log \frac{p(\theta|Z)}{p(\theta|Z')}$ under $p(\theta|Z)$.
By the tail bound to privacy conversion lemma (Lemma~\ref{lem:tailbound2DP}), we get 
that for any set $S\subset \Theta$
$\P(\theta \in S | Z) \leq e^\epsilon \P(\theta \in S | Z') +\delta$
for any $0<\delta<2/e$ and 
$$
\epsilon  = R\|\theta^*-\ttheta^*\|_2  + \log(1-\mu_2)  + \frac{\gamma c}{2}  + \frac{a}{2}  \log(2/\delta)  +  \frac{\gamma^{1/2} b}{2}  \sqrt{\log(2/\delta)}.
$$
Denote $v:=  \theta^*-\ttheta^*$, by strong convexity
$$\|v\|_2\leq \|\nabla l(y,x^T\theta)[\ttheta^*]\|_2/\alpha  = |l'| \|x\|_2 / \alpha \leq L\|x\|_2/\alpha$$
and 
$$
\|v\|_H \leq \|\nabla l(y,x^T\theta)[\ttheta^*]\|_{H^{-1}}  =  |l'| \|x\|_{H^{-1}} \leq L\|x\|_{H^{-1}}.
$$
Also use the fact that $|\log(1-\mu_2)| \leq 2\mu_2$ for $\mu_2<0.5$ and $\mu_2\leq \beta\|x\|_{\tilde{H}_2^{-1}}^2 $, we can then combine similar terms and have a more compact representation.
$$
\epsilon \leq \epsilon_0(1+\log(2/\delta))  +  e^{\frac{RL\|x\|_2}{\alpha}} \left[\frac{\gamma L^2\|x\|_{H^{-1}}^2}{2} +  \sqrt{ \gamma L^2\|x\|_{H^{-1}}^2\log(2/\delta) }\right]
$$
where 
$$\epsilon_0 \leq e^{\frac{RL\|x\|_2}{\alpha}} -1  + 2\beta \|x\|_{H_1^{-1}}^2 +  2\beta\|x\|_{\tilde{H}_2^{-1}}^2$$ 
is the part of the privacy loss that does not get smaller as $\gamma$ decreases.

\begin{proposition}\label{prop:generalnorm}
	Let $\|\cdot\|$ be a norm and $\|\cdot\|_*$ be its dual norm. Let $F(\theta)$, $f(\theta)$ and $\tilde{F}(\theta) = F(\theta) + f(\theta)$ be proper convex functions and $\theta^*$ and $\tilde{theta}^*$ be their minimizers, i.e., $0\in \partial F(\theta^*)$ and $0\in \partial \tilde{F}(\tilde{theta}^*)$.  If in addition, $F,\tilde{F}$ is $\alpha,\tilde{\alpha}$-strongly convex with respect to $\|\cdot\|$ within the restricted domain 
	$\theta \in \{  t\theta^* + (1-t)\tilde{\theta}^*  \;|\;  t\in[0,1]  \}$. 	Then there exists $g \in \partial f(\theta^*)$ and $\tilde{g}\in \partial f(\tilde{\theta}^*)$ such that
	$$
	\|\theta^*-\tilde{\theta}^*\| \leq\min\left\{\frac{1}{\alpha}\| \tilde{g}\|_*,  \frac{1}{\tilde{\alpha}}\| g\|_*\right\}.
	$$
\end{proposition}
\begin{proof}
	Apply the first order condition to $F$ restricted to the line segment between $\tilde{\theta}^*$ and $\theta^*$, there are we get
	\begin{align}
	F(\tilde{\theta}^*) \geq F(\theta^*)  +  \langle \partial F(\theta^*),  \tilde{\theta}^*-\theta^* \rangle  + \frac{\alpha}{2}\|\tilde{\theta}^*-\theta^*\|^2\label{eq:strongcvx1} \\
	F(\theta^*) \geq F(\tilde{\theta}^*)  +  \langle \partial F(\tilde{\theta}^*),  \theta^*-\tilde{\theta}^* \rangle  + \frac{\alpha}{2}\|\tilde{\theta}^*-\theta^*\|^2 \label{eq:strongcvx2}
	\end{align}
	Note by the convexity of $F$ and $f$, $\partial\tilde{F}=  \partial F + \partial f$, where $+$ is the Minkowski Sum. Therefore, $0\in \partial\tilde{F}(\tilde{\theta}^*)$ implies that there exists $\tilde{g}$ such that $\tilde{g}\in \partial f(\tilde{\theta}^*)$ and $-\tilde{g}\in\partial F(\tilde{\theta}^*)$.
	Take $-\tilde{g}\in\partial F(\tilde{\theta}^*)$ in Equation~\ref{eq:strongcvx2} and $0 \in \partial F(\theta^*)$ in Equation~\ref{eq:strongcvx1}  and add the two inequalities, we obtain
	$$
		0\geq \langle -\tilde{g},  \theta^*-\tilde{\theta}^* \rangle  + \alpha \|\tilde{\theta}^* - \theta^*\|^2 \geq - \|\tilde{g}\|_* \|\theta^*-\tilde{\theta}^*\|  +  \alpha\|\tilde{\theta}^* - \theta^*\|^2. 
	$$
	For $\|\tilde{\theta}^* - \theta^*\|=0$ the claim is trivially true, otherwise, we can divide the both sides of the above inequality by $\|\tilde{\theta}^* - \theta^*\|$ and get
	$	\|\theta^*-\tilde{\theta}^*\| \leq \frac{1}{\alpha}\| \tilde{g}\|_*$. 
	
	It remains to show that $\|\theta^*-\tilde{\theta}^*\| \leq \frac{1}{\tilde{\alpha}}\|g\|_*$. This can be obtained by exactly the same arguments above but applying strong convexity to $\tilde{F}$ instead. Note that we can actually get something slightly stronger than the statement because the inequality holds for all $g\in \partial f(\theta^*)$.
\end{proof}

A consequence of (generalized) self-concordance is the spectral (\emph{multiplicative}) stability of Hessian to small perturbations of parameters.
\begin{lemma}[Stability of Hessian{\citep[Theorem~2.1.1]{nesterov1994interior}, \citep[Proposition~1]{bach2010self}}]\label{lem:selfconcordant-hessian}
	Let $H_\theta :=  \nabla^2F_s(\theta)$. If $F_s$ is $R$-self-concordant at $\theta$. Then for any $v$ such that $R \|v\|_{H_\theta} < 1$, we have that
	$$
	(1-R\|v\|_{H_\theta})^2 \nabla^2 F_s(\theta) 	\prec	\nabla^2 F_s(\theta+v) \prec  \frac{1}{(1-R\|v\|_{H_\theta})^2}   \nabla^2 F_s(\theta)  .
	$$
	If instead we assume $F_s$ is $R$-generalized-self-concordant at $\theta$ with respect to norm $\|\cdot\|$, then
	$$
	e^{-R\|v\|} \nabla^2 F_s(\theta) \prec  \nabla^2 F_s(\theta+v)  \prec e^{R\|v\|}  \nabla^2 F_s(\theta) 
	$$
\end{lemma}\label{stability}
The two bounds are almost identical when  $R\|v\|$ and $R\|v\|_{\theta}$ are close to $0$, in particular, for $x\leq 1/2$, $e^{-2x} \leq 1-x \leq e^{-x}$.